\documentclass{article}

\usepackage[preprint]{main}
\usepackage{todonotes}

\usepackage[utf8]{inputenc} % allow utf-8 input
\usepackage[T1]{fontenc}    % use 8-bit T1 fonts
\usepackage{hyperref}       % hyperlinks
\usepackage{url}            % simple URL typesetting
\usepackage{booktabs}       % professional-quality tables
\usepackage{amsfonts}       % blackboard math symbols
\usepackage{nicefrac}       % compact symbols for 1/2, etc.
\usepackage{microtype}      % microtypography
\usepackage{xcolor}         % colors

\usepackage{amsmath}
\usepackage{amssymb}
\usepackage{mathtools}
\usepackage{amsthm}
\usepackage{bbold}
\usepackage{breqn}
\usepackage{subcaption}
\usepackage{multirow}
\usepackage{wrapfig}
\usepackage{amsmath}

\usepackage[capitalize,noabbrev]{cleveref}

\theoremstyle{plain}
\newcommand{\methodname}{Influence Distillation}
\newcommand{\expected}[1]{\mathop{\mathbb{E}[#1]}}

\newtheorem{theorem}{Theorem}[section]

\newtheorem{lemma}[theorem]{Lemma}

\newtheorem{corollary}[theorem]{Corollary}
\theoremstyle{definition}

\theoremstyle{remark}

\DeclareMathOperator*{\argmin}{arg\,min}

\newcommand{\blue}[1]{#1}

\newcommand\numberthis{\addtocounter{equation}{1}\tag{\theequation}}
\newcommand{\matr}[1]{\mathbf{#1}}
\newcommand{\vect}[1]{\mathbf{#1}}

\newcommand{\dS}{S}
\newcommand{\dT}{T}
\newcommand{\dD}{D}

\newcommand{\dV}{V}
\newcommand{\dL}{L}
\newcommand{\distS}{\mathcal{S}}
\newcommand{\distT}{\mathcal{T}}

\newcommand{\mech}{\mathcal{M}}
\newcommand{\loss}{\mathcal{L}}

\newcommand{\vtheta}{\boldsymbol{\theta}}
\newcommand{\vdelta}{\boldsymbol{\delta}}
\newcommand{\vw}{\boldsymbol{w}}
\newcommand{\vp}{\vect{p}}
\newcommand{\va}{\vect{a}}
\newcommand{\vb}{\vect{b}}

\newcommand{\vm}{\vect{m}}
\newcommand{\vg}{\vect{g}}
\newcommand{\vh}{\vect{h}}
\newcommand{\vv}{\vect{v}}
\newcommand{\vx}{\vect{x}}
\newcommand{\vy}{\vect{y}}
\newcommand{\ve}{\vect{e}}
\newcommand{\vt}{\vect{t}}
\newcommand{\vone}{\vect{\mathbb{1}}}
\newcommand{\vzero}{\vect{0}}
\newcommand{\valpha}{\boldsymbol{\alpha}}

\newcommand{\mG}{\matr{G}}
\newcommand{\mH}{\matr{H}}
\newcommand{\mQ}{\matr{Q}}
\newcommand{\mB}{\matr{B}}

\newcommand{\mC}{\matr{C}}
\newcommand{\mE}{\matr{E}}
\newcommand{\mI}{\matr{I}}
\newcommand{\mR}{\matr{R}}

\renewcommand{\paragraph}[1]{\noindent\textbf{#1}}

\newcommand{\Roned}[1]{\mathbb{R}^{#1}}
\newcommand{\Rtwod}[2]{\mathbb{R}^{#1 \times #2}}

\title{Efficient Data Selection at Scale \\ via Influence Distillation}

\author{%
  Mahdi Nikdan\thanks{Work done while an intern at Google Research.}~~\thanks{Correspondence to mnikdan@ista.ac.at and cohenaddad@google.com.}\\
  ISTA \& Google Research\\
  \And
  Vincent Cohen-Addad\footnotemark[2] \\
  Google Research \\
  \AND
  \quad Dan Alistarh\quad\quad~ \\
  \quad ISTA \& Red Hat AI\quad\quad~ \\
  \And
  ~Vahab Mirrokni\quad~ \\
  ~Google Research\quad~ \\
}

\begin{document}

\maketitle

\begin{abstract}
Effective data selection is critical for efficient training of modern Large Language Models (LLMs). This paper introduces \methodname{}, a novel, mathematically-justified framework for data selection that employs second-order information to optimally weight training samples. By distilling each sample's influence on a target distribution, our method assigns model-specific weights that are used to select training data for LLM fine-tuning, guiding it toward strong performance on the target domain. We derive these optimal weights for both Gradient Descent and Adam optimizers. To ensure scalability and reduce computational cost, we propose a \emph{landmark-based approximation}: influence is precisely computed for a small subset of ``landmark'' samples and then efficiently propagated to all other samples to determine their weights. We validate \methodname{} by applying it to instruction tuning on the Tulu V2 dataset, targeting a range of tasks including GSM8k, SQuAD, and MMLU, across several models from the Llama and Qwen families. Experiments show that \methodname{} matches or outperforms state-of-the-art performance while achieving up to $3.5\times$ faster selection.
\end{abstract}

\section{Introduction}

The rise of Large Language Models (LLMs) has driven significant advances in natural language processing; yet, training and fine-tuning these models requires massive computational resources and carefully-curated datasets. One key direction towards improved training efficiency has been via \emph{data selection and data weighting methods}~\citep{less, yin2024compute, antonello2020selecting, marion2023less, ankner2024perplexed, li2023quantity, ivison2025large, DBLP:conf/icml/AxiotisCHJMSWW24, dsir, dsdm, huang2024dynimpt}, which aim to curate training subsets that maximize a model's effectiveness, often with respect to a particular \emph{target data distribution} or downstream task. However, existing approaches typically rely on heuristics---such as perplexity-based filtering---or require expensive proxy model training or expensive embedding functions to generate data representations.

More precisely, existing methods face several limitations. First, many existing methods utilize fixed, model-agnostic features or representations (e.g., static embeddings) that may not capture the full relationship between training samples and the target distribution \citep{yin2024compute, antonello2020selecting, marion2023less, ankner2024perplexed}. 
Second, methods that update weights during training lack theoretical justification and can be unstable \citep{dsir, huang2024dynimpt}.
Lastly, approaches that rely on reference model training or costly embeddings are computationally intensive and often challenging to scale \citep{li2023quantity, less, ivison2025large}. Thus, there remains a clear need for a mathematically-grounded, efficient, and scalable framework for data selection that directly optimizes for performance on a specific target distribution.

\begin{wrapfigure}{t}{0.5\textwidth}
  \centering
  \includegraphics[width=0.99\linewidth]{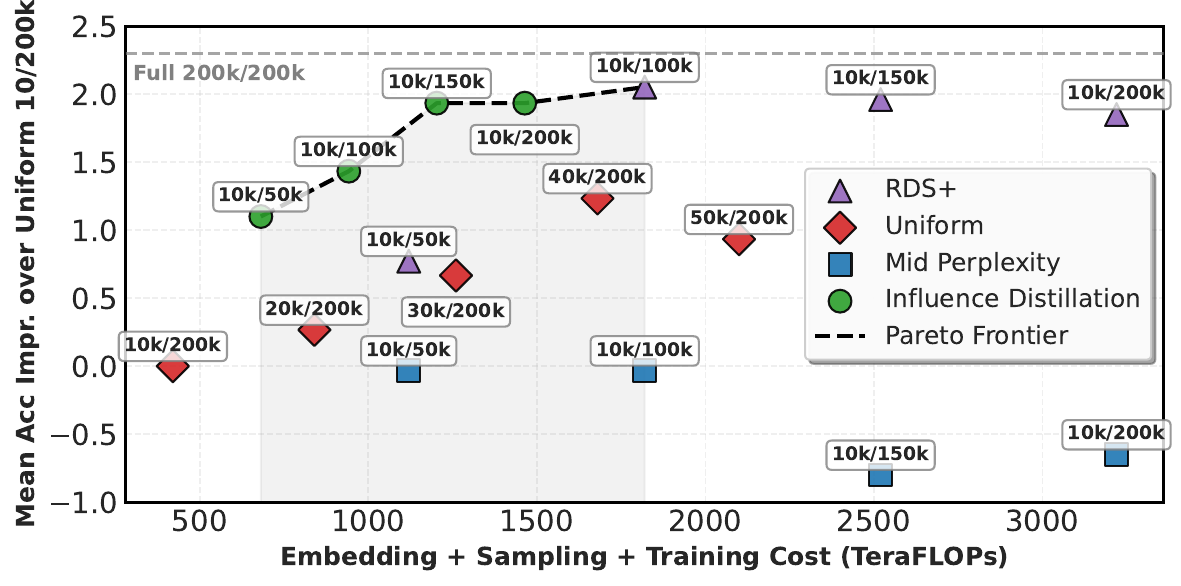}
  \caption{Average improvement over uniform sampling across six tasks vs. runtime. The model used is Llama2-7B~\citep{llama2}, and the training dataset is Tulu V2~\citep{tulu2}. The annotation ``M/N'' indicates that the method selected M samples from a pool of size N. Further details are provided in Section~\ref{sec:exp}.}
  \label{fig:pareto}
\end{wrapfigure}

\paragraph{Contribution.} We introduce \emph{\methodname{}}, a novel framework for data selection that addresses these challenges. Given a pre-trained model and a target task (represented by a small target dataset), \methodname{} formulates the influence of training samples on the target distribution's loss via a second-order approximation. \methodname{} directly optimizes sample weights by analyzing how each training sample, if included in a gradient step, is expected to affect model performance on the target data.
This formulation leads to a quadratic optimization objective for the sample weights, which we demonstrate can be solved efficiently. We provide derivations for these optimal weights under both standard Gradient Descent and the adaptive Adam optimizer, backed by theoretical justifications.

To ensure scalability to large datasets, we further introduce an efficient \emph{landmark-based} approximation. This approach first involves selecting a small subset of ``landmark'' samples and precisely computing their influence. The influence for all other samples is then efficiently approximated by transferring the computed influence from these landmarks. This transfer mechanism is guided by a novel and computationally inexpensive embedding space derived from Jacobian-vector Products. This significantly reduces the computational overhead of gradient computations for the entire dataset.

We validate \methodname{} via comprehensive instruction tuning experiments using standard open LLMs (from the Llama \citep{llama2, llama3} and Qwen \citep{qwen2.5} families) on the Tulu V2 \citep{tulu2} training dataset, while targeting advanced downstream tasks like MMLU \citep{mmlu1, mmlu2}, mathematics and code. Our results demonstrate that \methodname{} not only substantially outperforms uniform random selection but also in most cases matches or exceeds the performance of state-of-the-art data selection methods, while offering significant computational speedups on the same selection problem—up to 3.5$\times$ in embedding + selection runtime. This positions \methodname{} as a strong method on the Pareto frontier of overall embedding, selection and training cost versus downstream task accuracy (see Figure \ref{fig:pareto}).

\section{Related Work}

Data selection (`pruning') and weighting methods have become increasingly important in the context of efficient LLM training. In a celebrated paper,~\citet{sorscher2022beyond} 
et al. showed that (model-agnostic) 
data pruning, and in 
particular deduplication, helps go beyond 
scaling laws for LLMs. This was later further
improved by~\citet{abbas2023semdedup}.

Early work on model-dependent data pruning focused on heuristics like perplexity-based filtering and confidence-based selection:~\citet{marion2023less} found that selecting examples with moderate perplexity scores often outperforms training on the full dataset or examples selected by other metrics.  
\citet{do2019cross} introduced DSIR, which uses importance resampling based on n-gram features to select relevant training examples, with promising results on 
 mathematical reasoning and clinical text summarization. Similarly,~\citet{xie2023smalltolarge} proposed clustering loss trajectories to identify representative training examples, though their approach focused more on general domain adaptation rather than specific target distributions. 
 Another approach, so-called Classifier, was introduced by~\citet{brown2020language} and has been employed in subsequent work (\citet{gao2020pile,chowdhery2023palm,du2022glam}.  Other strategies include selecting examples that maximize the loss difference between LMs trained on candidate and reference datasets (\citet{moore2010intelligent,axelrod2017cynical,feng2022automatic}).  Simpler, yet common, techniques involve filtering documents based on length or the presence of excessive special characters (\citet{raffel2020exploring,xie2023smalltolarge}).  A related, though distinct, task in the LM domain is optimizing the weights for sampling from mixed data sources (\citet{chen2024skill,albalak2023efficient}). Recently, \citet{ivison2025large} proposed RDS+, which uses similarity between model-dependent embeddings computed by a position-weighted mean pool of the last hidden layer states.

Recent work has also highlighted the importance of considering the training dynamics when selecting data. \citet{zhou2023lobass} proposed measuring ``learnability'' based on loss changes during training, while~\citet{swayamdipta2020dataset} introduced ``dataset cartography'' to analyze training dynamics across examples. These methods provide useful signals about which examples are most valuable for training; at the same time, they require training reference models which can be computationally expensive.
For large-scale applications,~\citet{bhatt2024experimental} evaluated various data selection approaches for LLM fine-tuning, and found that facility-location selection based on hidden representations was particularly effective. However,~\citet{DBLP:conf/nips/TirumalaSAM23} observed that generating these representations for large datasets remains computationally challenging. 
More recently, \citet{DBLP:conf/icml/EngstromFM24} framed the data 
selection problem as an optimization problem: Given the learning algorithm, find the subset of the data that maximizes the performance
of the trained model. To obtain an efficient solution, they design a 
model that given a subset of the training data $S$ and a target example $t$, predicts the loss of the model trained on $S$ on $t$. 
\citet{DBLP:conf/icml/AxiotisCHJMSWW24} recently use coreset-related ideas to propose a 
computationally efficient way of sampling an unbiased estimator of the
model loss from the training data so as to train on a smaller input.

While previous methods like DSIR and facility location selection rely on fixed features or representations, our method directly optimizes sample weights based on their influence on the    target distribution through a second-order approximation. Importantly, this does not 
require training proxy model to predict the value of the elements
and is computed directly from the input, model and learning algorithm.
Unlike curriculum learning or confidence-based approaches that update weights during training, we derive optimal weights analytically for both SGD and Adam optimizers.
In contrast to methods that require training reference models, our landmark-based approximation allows efficient weight computation without extensive pre-training.

There is a large body of work on data selection methods for other 
learning tasks and mode, and it is beyond the scope of this paper to 
provide a detailed overview. We refer the reader to~\citet{kaushal2019learning,killamsetty2021retrieve,wei2015submodularity,chen2023alpagasus,cao2023instruction,sener2017active} and references therein.

\section{Method}
\label{sec:main-method}
\subsection{Problem and Notation}
Let $\vtheta \mathbin{\in} \Roned{d}$ be the model parameters. For any dataset $\dD$ of size $n$ and any vector of sample weights $\vw\mathbin{=}[w_1, w_2, ..., w_n]^T$, denote $\loss(\vtheta;\dD,\vw) \mathbin{=} \frac{1}{n} \sum_{i=1}^n w_i~\ell(\vtheta;\dD_i)$ as the weighted average of the model loss $\ell$ on the samples of dataset $\dD$ at point $\vtheta$. Additionally, define $\mech(\vtheta;\dD, \vw)$ as a training mechanism that returns the parameters after being trained on a dataset $\dD$ weighted by $\vw$. Unless otherwise stated, we will assume $\mech$ is simply one step of (full) gradient descent. 

Let $\distS$ and $\distT$ represent the training (source) and downstream (target) distributions, respectively. Assume we have access to a dataset $\dS$ sampled from $\distS$ and a small representative dataset $\dT$ from $\distT$. Our high-level goal will be to determine sample weights $\vw^*$ such that:
\begin{equation}
    \vw^* = \argmin_{\vw}~\loss \big(\mech(\vtheta; \dS, \vw);\dT, \vone\big)
    \label{eq:default-obj}
\end{equation}
where $\vone \mathbin{\in} \Roned{|\dT|}$ represents the all-ones vector. In words, we wish to find sample weights $\vw$ for instances within the source dataset $\dS$, such that training on $\dS$ using these weights results in minimal loss on the target dataset $\dT$. Notably, this notation also allows for the special case of $\distS \mathbin{=} \distT$, where our method would find weights that maximize in-distribution loss improvement.

\subsection{A Running Example}
\label{sec:running}
Throughout this section, we utilize a toy training setting to illustrate variants of our method. Specifically, we consider a linear regression model parameterized by $\vtheta$ with the loss function $\ell(\vtheta; \vect{x}, y) = (\vtheta^T \vect{x} - y)^2$ for any $\vtheta, \vect{x} \in \Roned{d}, y \in \{0,1\}$.
For the source dataset, we sample $256$ random instances from the first two classes of the CIFAR-10 dataset \citep{cifar10} and combine them with $256$ synthetic samples generated from a Gaussian distribution with the same mean and standard deviation as the real samples. The target dataset consists of another set of $256$ samples from CIFAR-10. We use gradient descent with a learning rate of $10^{-3}$ as the optimizer. Finally, the loss values are reported on a validation dataset of size $256$, also sampled from CIFAR-10.

\begin{figure}[t] % The * makes it span both columns
    \centering
    \begin{subfigure}{0.32\textwidth} % Each image takes up 1/2 of the width
        \includegraphics[width=\linewidth]{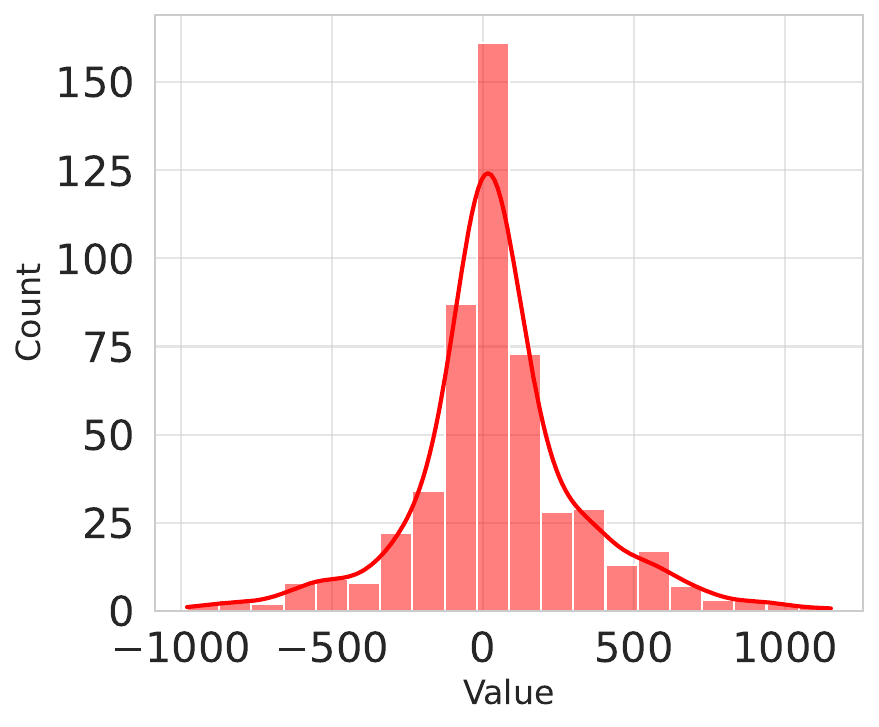}
    \end{subfigure}
    \hfill
    \begin{subfigure}{0.32\textwidth}
        \includegraphics[width=\linewidth]{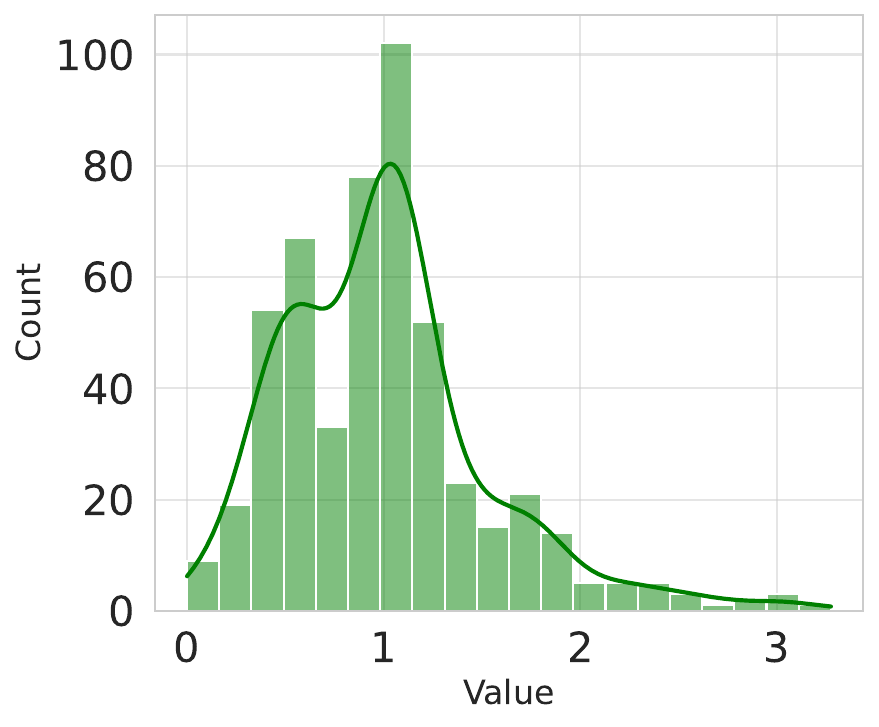}
    \end{subfigure}
    \begin{subfigure}{0.32\textwidth}
        \includegraphics[width=\linewidth]{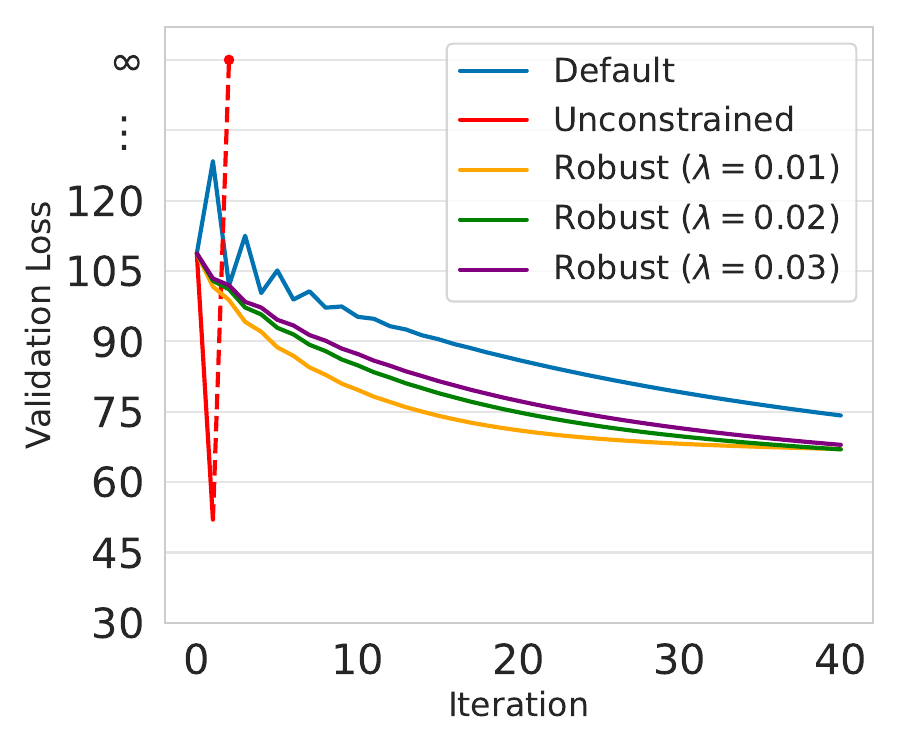}
    \end{subfigure}
    \caption{(Left) Distribution of unconstrained weights, (Middle) Distribution of robust weights for $\lambda\mathbin{=}0.02$, and (Right) validation loss during training with different variants in the running experiment setting. Robust weights are found by minimizing Objective \ref{eq:robust-obj} using the SLSQP algorithm \citep{slsqp} implemented in the SciPy library \citep{scipy}.}
    \label{fig:linear}
\end{figure}

\subsection{\methodname{}}
\paragraph{Case 1: Unconstrained Weights.}
Let $\vg_\dT(\vtheta) \mathbin{=} \nabla_{\vtheta} \loss(\vtheta;\dT,\mathbb{1})$ and $\mH_\dT(\vtheta) \mathbin{=} \nabla^2_{\vtheta} \loss(\vtheta;\dT,\mathbb{1})$ denote the gradient vector and Hessian matrix of the loss with respect to the model parameters on the target dataset. Construct $\mG_\dS(\vtheta) \in \Rtwod{|\dS|}{d}$ by stacking the gradients of the loss with respect to $\vtheta$ across samples of $\dS$. As mentioned before, assume $\mech$ is one step of gradient descent, i.e., $\mathcal{M(\vtheta;\dD}, \vw) \mathbin{=} \vtheta \mathbin{-} \eta \nabla_{\vtheta} \loss(\vtheta;\dD,\vw) \mathbin{=} \vtheta \mathbin{-} \frac{\eta}{|\dS|} \mG_\dS^T(\vtheta) \vw$, where $\eta$ denotes the learning rate.
We estimate Objective \ref{eq:default-obj} by:
\begin{align*}
    \vw^* &= \argmin_{\vw}~\loss \big(\mech(\vtheta; \dS, \vw);\dT, \mathbb{1}\big) \\
    &= \argmin_{\vw}~\loss (\vtheta - \frac{\eta}{|\dS|} \mG_\dS^T(\vtheta) \vw;\dT, \mathbb{1}) \\
    &\approx \argmin_{\vw}~[\loss (\vtheta;\dT, \mathbb{1}) - \frac{\eta}{|\dS|} \vg_\dT^T(\vtheta) \mG_\dS^T(\vtheta) \vw + \frac{\eta^2}{2~|\dS|^2} \vw^T \mG_\dS(\vtheta) \mH_\dT(\vtheta) \mG_\dS^T(\vtheta) \vw ] \\
    &= \argmin_{\vw}~ [ - \vg_\dT^T(\vtheta) \mG_\dS^T(\vtheta) \vw + \frac{\eta}{2~|\dS|} \vw^T \mG_\dS(\vtheta) \mH_\dT(\vtheta) \mG_\dS^T(\vtheta) \vw ] \numberthis \label{eq:taylor}
\end{align*}
where the approximation comes from a second-order Taylor expansion, i.e., $\loss(\vtheta+\vdelta;\dT,\vone)\mathbin{\approx}\loss(\vtheta;\dT,\vone) \mathbin{+} \vg_\dT^T(\vtheta)\vdelta \mathbin{+} \frac{1}{2}\vdelta^T \mH_\dT(\vtheta)\vdelta$ where $\vdelta$ is replaced with $-\frac{\eta}{|\dS|} \mG_\dS^T(\vtheta)\vw$.

Next, we define two key objects, $\vp \mathbin{\in} \Roned{|\dS|}$ and $\mQ \mathbin{\in} \Rtwod{|\dS|}{|\dS|}$, as follows:
\begin{equation}
\vp(\vtheta; \dS, \dT) = \mG_\dS(\vtheta) \vg_\dT(\vtheta),
\label{eq:p}
\end{equation}
\begin{equation}
\mQ(\vtheta; \dS, \dT) = \frac{1}{|\dS|} \mG_\dS(\vtheta) \mH_\dT(\vtheta) \mG_\dS^T(\vtheta),
\label{eq:Q}
\end{equation}
For brevity, unless stated otherwise, we will omit $\dS$ and $\dT$ from the arguments of $\vp$ and $\mQ$. 
Let 
\begin{equation}
\label{eq:f}
f(\vw;\vtheta) \mathbin{=} -\vp(\vtheta)^T \vw + \frac{\eta}{2} \vw^T \mQ(\vtheta) \vw.
\end{equation}
Then, the objective in Equation~\ref{eq:taylor} becomes
\begin{equation}
\vw^* = \argmin_{\vw}~ f(\vw;\vtheta).
\label{eq:local-obj}
\end{equation}
In words, $f$ represents a scaled approximation of the change in loss on $\dT$ when the model at point $\vtheta$ is trained on $\dS$ with weights $\vw$. It is a quadratic function in $\vw$, as $\vp$ and $\mQ$ do not depend on $\vw$. This objective can be minimized in closed form as $\vw^* \mathbin{=} \frac{1}{\eta} \mQ(\vtheta)^{-1} \vp(\vtheta)$.

\paragraph{Discussion.} While simple, the proposed solution has several crucial limitations: (a) it may produce negative or highly irregular sample weights, such as excessively large values, which lack intuitive interpretation, (b) the weights may overfit to the current set of parameters $\vtheta$, and (c) the weights may also overfit to the target dataset.
The first two issues can be easily observed in our running experiment. The irregularity of the weights is illustrated in Figure~\ref{fig:linear} (left). Furthermore, Figure~\ref{fig:linear} (right) demonstrates that unconstrained weights become invalidated after just one step of training, suggesting that the weights ``overfit'' to the current model parameters $\vtheta$. We note that, since the model in our running example is linear, the second-order approximation is exact—thus, the first update step reaches the optimum on $\dT$. However, this behavior does not generalize to non-linear models.

\paragraph{Case 2: Robust Weights.}
We modify Objective \ref{eq:local-obj} to address the above limitations. First, we restrict the weights to non-negative values, i.e., $\forall~ 1 \le i \le |\dS|: \vw_i \ge 0$. Second, we require the weights to sum to the size of the source dataset, $\vw^T \vone \mathbin{=} |\dS|$. This prevents weights from becoming excessively large and ensures that rescaling the weights does not change the effective step size: using $\alpha \vw$ with learning rate $\eta$ is equivalent to using $\vw$ with learning rate $\alpha \eta$. 

To mitigate ``overfitting'', a standard approach is to add a regularization term. Indeed, Appendix \ref{apx:linear}~derives such a term for linear models. In the general case, we employ a simple L2 regularization term.

\paragraph{The Robust Influence Objective.} 
Hence, we define the robust \methodname{} objective:
\begin{equation}
\vw^* = \argmin_{\vw}~ f(\vw;\vtheta) + \frac{\lambda}{2} \|\vw\|_2^2,~~s.t.~ \left\{\begin{matrix} \vw \ge 0 \\ \vw^T \vone = |\dS| \\ \end{matrix}\right.
\label{eq:robust-obj}
\end{equation}
Refer to Section \ref{sec:tuning-lambda} for a discussion on how we tune $\lambda$ in practice.

We compute the robust weights with $\lambda \mathbin{\in} \{0.01, 0.02, 0.03\}$ in the context of our running example. Figure~\ref{fig:linear} (right) highlights the effectiveness of these robust weights, showing that all three configurations outperform the default weights while remaining stable throughout training. Additionally, Figure~\ref{fig:linear} (middle) depicts the distribution of weights for $\lambda \mathbin{=} 0.02$.

\paragraph{Adam Optimizer.}
The Adam optimizer \citep{adam} is the default choice for fine-tuning LLMs. Therefore, we tailor our method for Adam optimizers. To this end, we employ a greedy approach, where we assume the first- and second-order momentums ($\vect{m}$ and $\vect{v}$, respectively) are fixed after a warm-up. In this case, the $\mQ^{\text{Adam}}$ and $\vp^{\text{Adam}}$ objects are calculated as follows:
\begin{equation}
    \label{eq:p-adam}
    \vp^{\text{Adam}}(\vtheta) = \mG_\dS^{\text{Adam}}(\vtheta) (\vg_\dT(\vtheta) - \frac{\eta}{|\dS|}~\mH_\dT(\vtheta) \vect{b}),
\end{equation}
\begin{equation}
    \label{eq:Q-adam}
    \mQ^{\text{Adam}}(\vtheta) = \frac{1}{|\dS|} \mG_\dS^{\text{Adam}}(\vtheta) \mH_\dT(\vtheta) \mG_\dS^{\text{Adam}}(\vtheta)^T,
\end{equation}
where $\vect{b}\mathbin{=}\frac{\beta_1 \vect{m}}{(1-\beta_1^s) (\sqrt{\frac{\vect{v}}{1-\beta_2^s}} + \epsilon)}$, and $\mG_\dS^{\text{Adam}}(\vtheta)$ is constructed by element-wise multiplying every row of $\mG_\dS(\vtheta)$ by $\vect{a}\mathbin{=}\frac{1-\beta_1}{(1-\beta_1^s) (\sqrt{\frac{\vect{v}}{1-\beta_2^s}} + \epsilon)}$. Additionally, $s$ is the number of warm-up steps, and $(\beta_1, \beta_2, \epsilon)$ are Adam hyperparameters. See Appendix \ref{apx:adam-greedy}~for more details.

\paragraph{Handling Variable Lebel Lengths.}
A common practice in data selection is to normalize the gradients prior to measuring similarities \citep{less}. This is motivated by the observation that the norm of a sample's gradient tends to correlate negatively with the number of label tokens, thereby biasing unnormalized gradient-based methods toward shorter samples. Normalizing the gradients mitigates this issue and shifts the similarity measure from a dot product to cosine similarity. In our approach, we adopt this normalization as well. See Appendix \ref{apx:grad-analysis}~for a study on this correlation.

\paragraph{Per-target Importance.}  
The formulation above assigns weights to training samples based on their \textit{average} influence over the target set. However, recent work~\citep{less, ivison2025large} has shown that selecting training samples based on \textit{per-target} influence can yield better performance; that is, iterating over individual target samples repeatedly and selecting one top-scoring training sample each time. We adopt this approach, noting that influence scores for each target can be computed by running \methodname{} $|\dT|$ times—once per target sample.

\section{Efficient \methodname{}}
\label{sec:impl}
In this section, we tackle several challenges regarding the implementation of \methodname{}.

\paragraph{Cost of Hessian.}
While $\mQ(\vtheta)$ can be calculated exactly using Hessian-Vector Products (HVPs), these HVPs require storing the backward graph, which can incur extra memory costs in practice.

\paragraph{Cost of $\mG_\dS$.}
Constructing the matrix $\mG_\dS$ requires computing the gradient of the model with respect to each individual sample in the training set. This process is computationally expensive, as it incurs a cost similar to one full epoch of training on the entire dataset $\dS$. Furthermore, storing the matrix $\mG_\dS$ requires memory proportional to $|\dS|$ times the size of the model, which is intractable in practice.

\paragraph{Regularization Coefficient} The solution to Equation~\ref{eq:robust-obj} is sensitive to the choice of regularization strength $\lambda$. A key challenge, therefore, is determining how to select $\lambda$ in a practical and effective way.

\subsection{First-order \methodname{}}
\blue{Recall the definition of $f(\vtheta; \vw)$ from Equation~\ref{eq:f}, where the second-order term is scaled by the learning rate $\eta$. In Appendix \ref{apx:first-vs-second-order}, we observe that in our settings, $\eta$ is small enough that the second-order term becomes negligible. As a result, computing $\mQ$ can be avoided with little to no loss in performance. This first-order approximation aligns with prior work on gradient-based influence estimation, such as the methods proposed by \citet{less}.}

\subsection{Gradient Projection}
To reduce the cost of storing the gradients, we take a similar approach to \citet{less} and project each gradient vector upon calculation into a $k$-dimensional space, where $k \ll d$. As opposed to the mentioned work, which uses random projections sampled from the Rademacher distribution ($\pm1$ with equal probability), we find that projection using a Randomized Hadamard Transform is faster in practice. For more details on the projections, see Appendix \ref{apx:proj}.

\subsection{Landmark-Based Gradient Approximation}
\label{sec:landmark}
To circumvent the need to compute gradients for every training sample, we introduce a \textit{landmark-based} approach. At a high level, this method provides an efficient low-rank approximation of the gradient matrix $\mG_\dS$, given by $\hat{\mG}_\dS = \mC \mG_\dL$, where $\mG_\dL \in \Rtwod{\ell}{d}$ contains the gradients of $\ell \ll |\dS|$ selected \textit{landmark} samples. The matrix $\mC \in \Rtwod{|\dS|}{\ell}$ holds the coefficients that express each sample’s gradient as a linear combination of the landmark gradients.

Specifically, let $\dL \subseteq \dS$ denote a set of $\ell$ landmark samples (e.g., selected at random), and suppose we have access to low-cost per-sample embeddings, represented by $\mE_\dS \in \Rtwod{|\dS|}{e}$. As before, we assume that all embedding and gradient vectors are normalized. To compute the coefficient matrix $\mC$, we minimize the objective $\min_{\mC \in \Rtwod{n}{\ell}} ||\mE_\dS - \mC \mE_\dL||_2^2$, where $\mE_\dL \in \Rtwod{\ell}{e}$ contains the embeddings of the landmark samples. In words, this procedure approximates each sample’s embedding as a linear combination of landmark embeddings. We then estimate the $i$-th row of the gradient matrix, $\vg_i$, by $\hat{\vg}_i = \mG_\dL^T \vect{c}_i$, where $\vect{c}_i$ is the $i$-th row of $\mC$. This approximation implicitly assumes that the linear relationships learned in the embedding space transfer to the gradient space.

\paragraph{Theoretical Justification.}
Although this approximation is not expected to recover the true gradients with high accuracy, the key intuition is that, as long as it is unbiased, even a weak recovery can yield similar per-sample weights in high-dimensional spaces. Theorem \ref{cla:approx-grads} demonstrates this property for the specific case of the first-order variant of \methodname{}.
\begin{theorem}
    \label{cla:approx-grads}
    \blue{(Informal version of Theorem \ref{thm:approx-grad} and Corollary \ref{cor:approx-grads} tailored to landmark-based approximation -- see Appendix \ref{apx:cluster-bound})} Consider the special case of first-order \methodname{}. Let $\vg_i$ and $\hat{\vg}_i$ denote the true and the landmark-based approximated gradients for sample $i$, and assume \methodname{} with $\mG_\dS$ and $\hat{\mG}_\dS$ results in sample weights of $\vw$ and $\hat{\vw}$. Further assume:
    \begin{itemize}
        \item Unbiased: $\forall i \in \{1,2,\dots,n\}: \expected{\hat{\vg}_i} = \vg_i$, i.e., the approximation is unbiased.
        \item Bounded Low-rank MSE: Let $\delta_i^2 = \mathbb{E}[||\hat{\mathbf{g}}_i-\mathbf{g}_i||^2]$, and for some $\Delta^2 \ge 0$: $\frac{1}{n} \sum_{i=1}^{n} \delta_i^2 \le \Delta^2$.
    \end{itemize}
    Then $\expected{||\vw - \hat{\vw})||^2} \le \frac{|\dS| \Delta^2}{\lambda^2 d}$, with $\lambda$ being the regularization coefficient in \methodname{}.
\end{theorem}

This theorem relates the accuracy of the weights to the low-rank approximation error $\Delta^2$, given the training set size $|\dS|$, dimension $d$, and regularization parameter $\lambda$. If the approximations are unbiased, in high dimension $d$, it suffices to reasonably control $\Delta^2$ in order to recover the correct weights.

\paragraph{Integration with \methodname{}.}
Given a low-rank approximation of the gradient matrix $\mG_\dS \approx \hat{\mG_\dS} = \mC \mG_\dL$, one can define approximations to objects $\vp$ and $\mQ$ as below:

\begin{equation}
\hat{\vp}(\vtheta;\dS, \dT, \dL) = \mC \mG_\dL(\vtheta) \vg_\dT(\vtheta) = \mC \cdot \vp(\vtheta;\dL, \dT),
\label{eq:phat}
\end{equation}
\begin{equation}
\hat{\mQ}(\vtheta;\dS, \dT, \dL) = \frac{1}{|\dS|} \mC \mG_\dL \mH_\dT(\vtheta) \mG_\dL^T \mC^T(\vtheta) = \frac{|\dL|}{|\dS|} \mC \cdot \mQ(\vtheta;\dL, \dT) \cdot \mC^T,
\label{eq:Qhat}
\end{equation}
where $\mC=\min_{\mC \in \Rtwod{n}{\ell}} ||\mE_\dS - \mC \mE_\dL||_2^2$ is the coefficient matrix, defined above. As shown in Equations \ref{eq:phat} and \ref{eq:Qhat}, the landmark-based \methodname{} computes $\vp$ and $\mQ$ only for the landmark points, and then propagates them to the remaining samples.

\paragraph{JVP Embeddings.}
We observe that existing embedding methods perform poorly in this setting, exhibiting weak correlation with the true gradients (see Appendix \ref{apx:embds}~for a detailed empirical analysis). To address this issue, we introduce \textit{Jacobian-vector Product (JVP) Embeddings}.

Given a sample $x \in \dS$, we define its JVP embedding as:
\begin{equation}
h_{JVP}(x; \mathcal{N}, \ell, \dV) = \frac{1}{|\dV|} \sum_{\vv \in \dV} \frac{\partial \mathcal{N}_\ell(x)}{\partial \vtheta_\ell} \cdot \vv
\end{equation}
where $\mathcal{N}$ is the model being trained, $\mathcal{N}_\ell(\cdot)$ represents the logits of the next predicted token after processing $x$ through the first $\ell$ layers (or transformer blocks, in case of LLMs), and $\vtheta_\ell$ are the parameters of these $\ell$ layers. The set $\dV$ contains random Gaussian vectors matching the shape of $\vtheta_\ell$, and the term $\frac{\partial \mathcal{N}_\ell(x)}{\partial \theta_\ell}$ is the Jacobian of $\mathcal{N}_\ell(x)$ with respect to $\theta_\ell$. In words, JVP embeddings project the Jacobian of an intermediate model output onto a set of random directions in parameter space.

\subsection{Tuning the Regularization Coefficient.}
\label{sec:tuning-lambda}
Finally, we describe our approach for selecting the regularization coefficient $\lambda$ in Equation~\ref{eq:robust-obj}. As detailed in Appendix \ref{apx:closed-form}, when $\eta$ is small enough for the second-order term to be negligible and $\lambda = 0$, the solution assigns all the weight to a single sample. As $\lambda$ increases, the solution becomes progressively less sparse, distributing weight across more samples. In the limit $\lambda \rightarrow \infty$, the solution becomes fully dense, assigning equal weight to all samples. In practice, given a budget of $k$ samples to select for training, we tune $\lambda$ via binary search to achieve a target sparsity level of $\frac{|\dS| - k}{|\dS|}$, thereby ensuring that exactly $k$ samples receive non-zero weight, which we will pick for training.

\section{Experiments}
\label{sec:exp}
In this section, we evaluate \methodname{} across several challenging tasks. We start by detailing the datasets, models, and hyperparameters used in our experiments. Then we present our main results and ablations. Further studies are included in Appendix.

\subsection{Setting}
\label{sec:exp-setting}
We largely follow the experimental setup of \citet{ivison2025large} and reuse their code. 

\paragraph{Training Dataset.}
We use Tulu V2 \citep{tulu2}, a combination of 9 instruction-tuning datasets containing approximately 5.8 million samples. Detailed descriptions of each component dataset are provided in Appendix \ref{apx:datasets-models}. Unless stated otherwise, we randomly sample 200k examples from Tulu V2, and then use sampling methods to pick a subset of 10k samples from this pool. 

\paragraph{Target Datasets.}
We evaluate on six target datasets: MMLU \citep{mmlu1, mmlu2}, GSM8k \citep{gsm8k}, BBH \citep{bbh}, TyDIQA \cite{tydiqa}, Codex \citep{codex}, and SQuAD \citep{squad}. For each, we assume access to 8–500 examples from their train, dev, or eval splits \footnote{\blue{Following \citet{ivison2025large}, except we omit AlpacaEval \citep{li2023alpacaeval} as it requires paid  API access.}}. Details are in Appendix \ref{apx:datasets-models}.

\paragraph{Model.}
We mainly consider fine-tuneing the LlaMA-2 7B model \citep{llama2}, consistent with the Tulu V2 paper \citep{tulu2} and the experiments of \citet{ivison2025large}. We also experiment with Llama-3.2 3B \citep{llama3} and Qwen 2.5 1.5/3B \citep{qwen2.5}.

\paragraph{Baselines.} We consider four baselines: (1) Random \textit{Uniform} selection, which picks samples uniformly at random, (2) The state-of-the-art \textit{RDS+} \citep{ivison2025large} embedding-based method, where the embeddings are computed by a position-weighted mean pool of the last hidden layer states, (3) \textit{Mid-PPL} \citep{yin2024compute}, where samples are sorted by their perplexity, and the middle ones are selected, and (4) \textit{Full}, where we do not perform any sampling and train on the full dataset. \blue{Additionally, in Appendix \ref{apx:embds}, we evaluate \methodname{} using the true gradients as embeddings, which we show corresponds to the expensive LESS method \citep{less}.}

\paragraph{Hyperparameters.}
We use the AdamW optimizer with a learning rate of $2 \times 10^{-5}$ and a linear schedule for 2 epochs. The sequence length is fixed at 2048, and we use a microbatch size of 1 with gradient accumulation over 128 steps. All experiments are conducted on a single H100 GPU, and each are repeated with 3 random seeds, including the selection of 200k samples from Tulu V2.

By default, we use first-order \methodname{} with 4096 landmarks. We select the landmarks uniformly at random, as we find this performs comparably to more complex methods such as leverage score sampling or clustering. Linear coefficients are computed via Kernel Ridge Regression (KRR) with an RBF kernel and dampening of $0.01$. JVP embeddings are obtained from the first four transformer blocks using two random vectors ($\ell = 4$, $|\dV| = 2$), following a brief warm-up on 10k random samples. The model is then reset and trained on the selected subset. This warm-up is needed to stabilize gradients (see Appendix \ref{apx:grad-analysis}). Gradients are projected to 131072 dimensions via Hadamard projections; we use the largest dimension that fits in GPU memory, as projection cost does not depend on the dimension (Appendix \ref{apx:proj}). After selection, we do not incorporate the sample weights during training, as experiments in Appendix \ref{apx:weighted-loss} suggest this does not improve performance.
\subsection{Results}
\begin{table*}[t]
\centering
\caption{Accuracy ($\pm$ standard deviation) and estimated embedding and selection cost (Embd+Sel, in TF, TeraFLOPs) of various methods across tasks and models. For each model–dataset pair, 10k training samples are selected from a pool of size $200k$ from the Tulu V2 dataset \citep{tulu2}. We additionally report average improvement over the Uniform baseline (Avg. $\Delta$ w/ Uniform). Top performing selection methods, as well as Full training numbers are in bold.}
\scalebox{0.60}{
\begin{tabular}{l|lcccccc|c|c}
\toprule
\textbf{Model} & \textbf{Method} & \textbf{MMLU} & \textbf{GSM8k} & \textbf{BBH} & \textbf{TyDIQA} & \textbf{CODEX} & \textbf{SQuAD} & \textbf{Avg. $\Delta$ w/ Uniform} & \textbf{Embd+Sel Cost} \\
\midrule
\multirow{4}{*}{Llama2-7B} 
  & Uniform   & 45.6 $\pm$ 0.43 & 17.5 $\pm$ 1.08 & 41.8 $\pm$ 0.20 & 51.6 $\pm$ 0.38 & 27.0 $\pm$ 0.60 & 80.8 $\pm$ 1.05 & 0.00 & 0 \\
  & Mid-PPL  & 45.6 $\pm$ 0.86 & 15.0 $\pm$ 0.54 & 40.9 $\pm$ 0.23 & 52.1 $\pm$ 0.44 & 26.1 $\pm$ 1.56 & 80.7 $\pm$ 0.73 & -0.65 & 2800 TF \\
  & RDS+     & 46.3 $\pm$ 0.33 & 20.2 $\pm$ 2.77 & 42.7 $\pm$ 0.61 & 50.5 $\pm$ 0.84 & \textbf{30.4 $\pm$ 0.96} & \textbf{85.3 $\pm$ 0.22} & +1.85 & 2800 TF \\
  & InfDist  & \textbf{48.3 $\pm$ 0.21} & \textbf{20.3 $\pm$ 1.65} & \textbf{43.2 $\pm$ 0.67} & \textbf{53.6 $\pm$ 0.34} & 29.5 $\pm$ 3.14 & 83.2 $\pm$ 1.02 & \textbf{+2.30} & \textbf{872 TF} \\
  & \textbf{Full}     & \textbf{48.8 $\pm$ 0.08} & \textbf{21.2 $\pm$ 0.85} & \textbf{43.9 $\pm$ 0.24} & \textbf{51.3 $\pm$ 0.18} & \textbf{29.3 $\pm$ 3.72} & \textbf{83.6 $\pm$ 0.30} & \textbf{+2.30}  & $-$ \\
\midrule
\multirow{4}{*}{Llama3.2-3B} 
  & Uniform   & 53.9 $\pm$ 0.52 & 34.6 $\pm$ 1.22 & 48.9 $\pm$ 0.67 & 63.1 $\pm$ 0.36 & 56.1 $\pm$ 1.35 & 80.4 $\pm$ 0.51 & 0.00  & 0 \\
  & Mid-PPL  & \textbf{54.0 $\pm$ 0.27} & 29.5 $\pm$ 0.12 & 48.3 $\pm$ 0.44 & \textbf{65.9 $\pm$ 0.66} & 55.9 $\pm$ 4.40 & 80.9 $\pm$ 0.18 & -0.42 & 1200 TF \\
  & RDS+     & 53.1 $\pm$ 0.58 & \textbf{38.4 $\pm$ 0.58} & \textbf{49.6 $\pm$ 0.45} & 61.0 $\pm$ 0.35 & \textbf{60.6 $\pm$ 1.77} & \textbf{84.2 $\pm$ 0.47} & \textbf{+1.65}  & 1200 TF \\
  & InfDist  & \textbf{54.0 $\pm$ 0.94} & 35.7 $\pm$ 1.28 & 48.6 $\pm$ 0.27 & 64.6 $\pm$ 1.29 & 55.4 $\pm$ 1.10 & 83.3 $\pm$ 1.97 & +0.77  & \textbf{417 TF} \\
  & \textbf{Full}     & \textbf{52.9 $\pm$ 0.87} & \textbf{37.0 $\pm$ 0.33} & \textbf{48.9 $\pm$ 0.14} & \textbf{62.5 $\pm$ 1.50} & \textbf{57.7 $\pm$ 2.81} & \textbf{83.9 $\pm$ 1.15} & \textbf{+0.98}  & $-$ \\
\midrule
\multirow{4}{*}{Qwen2.5-1.5B} 
  & Uniform   & 58.8 $\pm$ 0.11 & \textbf{63.3 $\pm$ 1.67} & 44.1 $\pm$ 0.38 & 55.0 $\pm$ 0.32 & 70.5 $\pm$ 2.06 & 16.5 $\pm$ 4.65 & 0.00  & 0 \\
  & Mid-PPL  & 58.9 $\pm$ 0.17 & 63.2 $\pm$ 0.65 & \textbf{44.3 $\pm$ 0.31} & 54.3 $\pm$ 0.29 & 70.6 $\pm$ 1.61 & 22.3 $\pm$ 3.39 & +0.90 & 600 TF \\
  & RDS+     & 58.3 $\pm$ 0.07 & 60.1 $\pm$ 0.13 & 44.2 $\pm$ 0.24 & 53.0 $\pm$ 0.38 & \textbf{72.3 $\pm$ 0.00} & 46.0 $\pm$ 3.12 & +4.28  & 600 TF \\
  & InfDist  & \textbf{59.4 $\pm$ 0.12} & 62.0 $\pm$ 0.46 & 44.1 $\pm$ 0.35 & \textbf{57.6 $\pm$ 0.32} & 69.8 $\pm$ 1.15 & \textbf{54.4 $\pm$ 13.13} & \textbf{+6.52}  & \textbf{208 TF} \\
  & \textbf{Full}     & \textbf{59.4 $\pm$ 0.13} & \textbf{60.3 $\pm$ 0.38} & \textbf{44.0 $\pm$ 0.22} & 5\textbf{0.3 $\pm$ 1.21} & \textbf{73.0 $\pm$ 1.79} & \textbf{63.2 $\pm$ 6.05} & \textbf{+7.00}  & $-$ \\
\midrule
\multirow{4}{*}{Qwen2.5-3B} 
  & Uniform   & 63.7 $\pm$ 0.27 & 68.7 $\pm$ 1.87 & 54.9 $\pm$ 0.24 & 65.6 $\pm$ 0.16 & 83.1 $\pm$ 1.35 & 84.5 $\pm$ 0.54 & 0.00  & 0 \\
  & Mid-PPL  & 63.7 $\pm$ 0.18 & \textbf{70.8 $\pm$ 0.84} & \textbf{55.1 $\pm$ 0.21} & 65.3 $\pm$ 0.41 & \textbf{83.3 $\pm$ 3.14} & 79.9 $\pm$ 2.61 & -0.40 & 1200 TF \\
  & RDS+     & 63.6 $\pm$ 0.19 & 67.4 $\pm$ 0.68 & 54.0 $\pm$ 0.63 & 65.1 $\pm$ 0.33 & 82.4 $\pm$ 1.46 & \textbf{86.3 $\pm$ 0.33} & -0.28  & 1200 TF \\
  & InfDist  & \textbf{64.6 $\pm$ 0.19} & 67.8 $\pm$ 0.60 & 53.9 $\pm$ 0.36 & \textbf{66.9 $\pm$ 0.23} & 82.4 $\pm$ 0.55 & 86.0 $\pm$ 0.22 & \textbf{+0.18}  & \textbf{340 TF} \\
  & \textbf{Full}     & \textbf{63.8 $\pm$ 0.06} & \textbf{71.0 $\pm$ 1.78} & \textbf{53.8 $\pm$ 0.32} & \textbf{64.5 $\pm$ 0.42} & \textbf{82.0 $\pm$ 1.41} & \textbf{85.4 $\pm$ 0.42} & \textbf{0.00}  & $-$ \\
\bottomrule
\end{tabular}
}
\label{tab:main-exp}
\end{table*}
\paragraph{Main Experiments.}
Table~\ref{tab:main-exp} summarizes our main experimental results. In each case, a subset of size 10k is selected from a pool of 200k Tulu V2~\citep{tulu2} samples. On average, \methodname{} achieves higher performance compared to other more expensive selection baselines in three out of four models and remains competitive in the fourth, while enabling 2.9--3.5$\times$ faster selection. Notably, for two models, it matches or surpasses training on the full dataset. These results clearly showcase the effectiveness and efficiency of \methodname{}.

\paragraph{Selection Runtime Estimation.}
The table also reports the approximate FLOPs required for sample selection. Following the estimation from \citet{kaplan2020scaling}, each forward pass is assumed to cost $2d$ FLOPs and each backward pass $4d$ FLOPs, where $d$ is the number of model parameters. Mid-PPL and RDS+ require one forward pass per sample. \methodname{} requires computing a JVP embedding for each sample, along with full forward and backward passes for the 4096 selected landmarks. We estimate the cost of a JVP as $2\times$ that of a partial forward pass over the same number of blocks, following \citet{jvp1}.

\paragraph{Pareto Superiority.}
We repeat the above experiments on the Llama2-7B model using pool sizes 50k, 100k, 150k, and 200k. For each pool size, we select 2048, 4096, 6144, and 8192 landmarks, respectively, maintaining a fixed pool-to-landmark ratio. As shown in Figure~\ref{fig:pareto}, all points corresponding to \methodname{} lie on the Pareto frontier, matching or surpassing the performance of RDS+ while requiring lower overall cost of embedding, sampling, and training.

\paragraph{Effect of Number of Landmarks.}
To evaluate how the number of landmarks impacts performance, we repeat the experiments on Llama2-7B and report average improvement over the Uniform baseline across the six tasks in Figure~\ref{fig:landmarks_heatmap} (Left). As shown, \methodname{} improves with more landmarks, surpassing RDS+ beyond 2048 landmarks.

\paragraph{Effect of Pool Size and Number of Selected Samples.}
Figure~\ref{fig:landmarks_heatmap} (Right) presents a heatmap of MMLU accuracy on Llama2-7B across different combinations of pool size (up to 200k) and number of selected samples. For each pool size, we use the same number of landmarks as the Pareto experiment. As expected, accuracy improves with larger pools and more selected samples, highlighting the scalability and robustness of \methodname{}.

\begin{figure}[t] % The * makes it span both columns
    \centering
    \begin{subfigure}{0.47\textwidth} % Each image takes up 1/2 of the width
        \includegraphics[width=\linewidth]{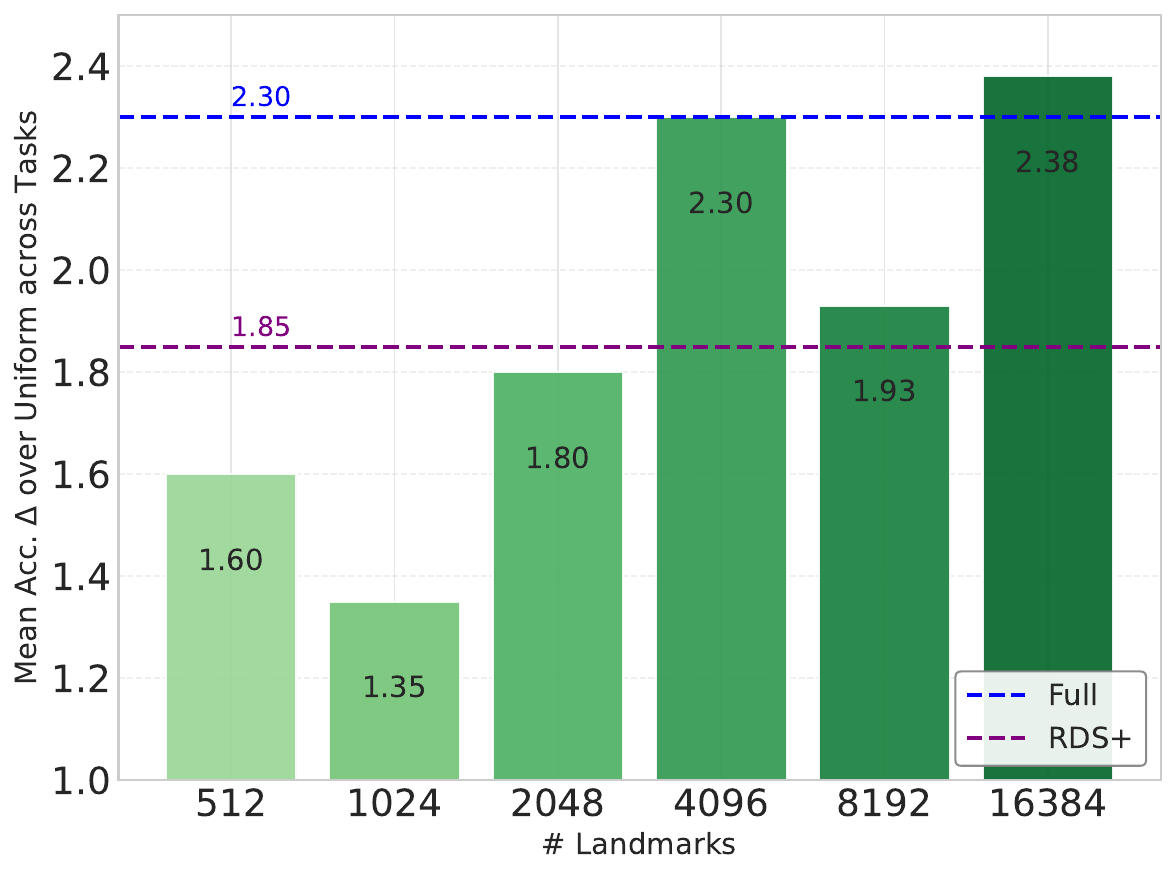}
    \end{subfigure}
    \hfill
    \begin{subfigure}{0.47\textwidth}
        \includegraphics[width=\linewidth]{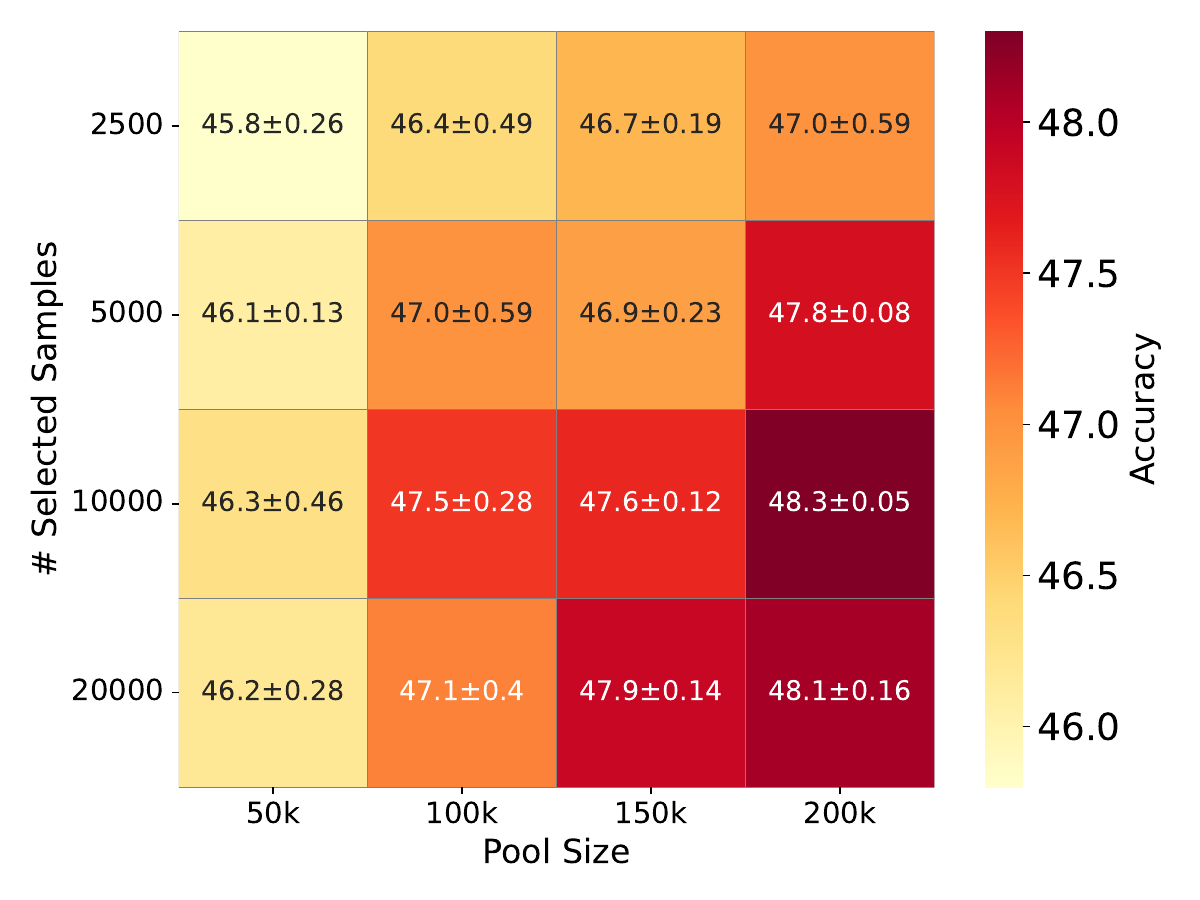}
    \end{subfigure}
    \caption{(Left) Effect of the number of landmarks on the performance of \methodname{} across six tasks using Llama2-7B. (Right) MMLU accuracy of \methodname{} on Llama2-7B across different pool sizes and number of selected samples.}
    \label{fig:landmarks_heatmap}
\end{figure}

\section{Limitations and Future Work}
\label{sec:limitation}
Below, we outline three main limitations of our work, along with corresponding directions for future research.

\paragraph{No Target Distribution.}
\blue{While we demonstrate that \methodname{} is highly effective for targeted instruction tuning across a range of models and tasks, it does not directly extend to general data selection scenarios where no target dataset is available. In such cases, one could define the target distribution as a small set of high-quality examples or a representative subset of the training corpus. Investigating how to construct and utilize such proxy targets is an important direction for future work.}

\paragraph{Pre-training.}
\blue{Extending \methodname{} to the pre-training setting presents unique challenges. In particular, the considerably longer duration of pre-training implies that gradients may shift substantially over time, likely making a single static selection insufficient. This suggests the need for a multi-phase selection strategy, such as periodic re-sampling. We leave the exploration of such dynamic approaches to future work.}

\paragraph{Warm-up Cost.}
\blue{We exclude the cost of the warm-up phase from our runtime measurements for two reasons: (1) as the training pool grows, the cost of a brief warm-up on a small random subset becomes negligible compared to embedding the full dataset; and (2) the warm-up can be shortened (as shown in Appendix \ref{apx:grad-analysis}) or compressed—for example, via Low-Rank Adaptation \citep{hu2022lora}, as in \citet{less}. We leave a rigorous investigation of warm-up optimization to future work.}

\bibliography{example_paper}

\begin{thebibliography}{73}
\providecommand{\natexlab}[1]{#1}
\providecommand{\url}[1]{\texttt{#1}}
\expandafter\ifx\csname urlstyle\endcsname\relax
  \providecommand{\doi}[1]{doi: #1}\else
  \providecommand{\doi}{doi: \begingroup \urlstyle{rm}\Url}\fi

\bibitem[Xia et~al.(2024)Xia, Malladi, Gururangan, Arora, and Chen]{less}
Mengzhou Xia, Sadhika Malladi, Suchin Gururangan, Sanjeev Arora, and Danqi Chen.
\newblock Less: Selecting influential data for targeted instruction tuning.
\newblock \emph{arXiv preprint arXiv:2402.04333}, 2024.

\bibitem[Yin and Rush(2024)]{yin2024compute}
Junjie~Oscar Yin and Alexander~M Rush.
\newblock Compute-constrained data selection.
\newblock \emph{arXiv preprint arXiv:2410.16208}, 2024.

\bibitem[Antonello et~al.(2020)Antonello, Beckage, Turek, and Huth]{antonello2020selecting}
Richard Antonello, Nicole Beckage, Javier Turek, and Alexander Huth.
\newblock Selecting informative contexts improves language model finetuning.
\newblock \emph{arXiv preprint arXiv:2005.00175}, 2020.

\bibitem[Marion et~al.(2023)Marion, {\"U}st{\"u}n, Pozzobon, Wang, Fadaee, and Hooker]{marion2023less}
Max Marion, Ahmet {\"U}st{\"u}n, Luiza Pozzobon, Alex Wang, Marzieh Fadaee, and Sara Hooker.
\newblock When less is more: Investigating data pruning for pretraining llms at scale.
\newblock \emph{arXiv preprint arXiv:2309.04564}, 2023.

\bibitem[Ankner et~al.(2024)Ankner, Blakeney, Sreenivasan, Marion, Leavitt, and Paul]{ankner2024perplexed}
Zachary Ankner, Cody Blakeney, Kartik Sreenivasan, Max Marion, Matthew~L Leavitt, and Mansheej Paul.
\newblock Perplexed by perplexity: Perplexity-based data pruning with small reference models.
\newblock \emph{arXiv preprint arXiv:2405.20541}, 2024.

\bibitem[Li et~al.(2023{\natexlab{a}})Li, Zhang, Li, Chen, Chen, Cheng, Wang, Zhou, and Xiao]{li2023quantity}
Ming Li, Yong Zhang, Zhitao Li, Jiuhai Chen, Lichang Chen, Ning Cheng, Jianzong Wang, Tianyi Zhou, and Jing Xiao.
\newblock From quantity to quality: Boosting llm performance with self-guided data selection for instruction tuning.
\newblock \emph{arXiv preprint arXiv:2308.12032}, 2023{\natexlab{a}}.

\bibitem[Ivison et~al.(2025)Ivison, Zhang, Brahman, Koh, and Dasigi]{ivison2025large}
Hamish Ivison, Muru Zhang, Faeze Brahman, Pang~Wei Koh, and Pradeep Dasigi.
\newblock Large-scale data selection for instruction tuning.
\newblock \emph{arXiv preprint arXiv:2503.01807}, 2025.

\bibitem[Axiotis et~al.(2024)Axiotis, Cohen{-}Addad, Henzinger, Jerome, Mirrokni, Saulpic, Woodruff, and Wunder]{DBLP:conf/icml/AxiotisCHJMSWW24}
Kyriakos Axiotis, Vincent Cohen{-}Addad, Monika Henzinger, Sammy Jerome, Vahab Mirrokni, David Saulpic, David~P. Woodruff, and Michael Wunder.
\newblock Data-efficient learning via clustering-based sensitivity sampling: Foundation models and beyond.
\newblock In \emph{Forty-first International Conference on Machine Learning, {ICML} 2024, Vienna, Austria, July 21-27, 2024}. OpenReview.net, 2024.
\newblock URL \url{https://openreview.net/forum?id=WUQ4YzIQt2}.

\bibitem[Xie et~al.(2023{\natexlab{a}})Xie, Santurkar, Ma, and Liang]{dsir}
Sang~Michael Xie, Shibani Santurkar, Tengyu Ma, and Percy~S Liang.
\newblock Data selection for language models via importance resampling.
\newblock \emph{Advances in Neural Information Processing Systems}, 36:\penalty0 34201--34227, 2023{\natexlab{a}}.

\bibitem[Engstrom et~al.()Engstrom, Feldmann, and Madry]{dsdm}
Logan Engstrom, Axel Feldmann, and Aleksander Madry.
\newblock Dsdm: Model-aware dataset selection with datamodels, 2024.
\newblock \emph{URL https://arxiv. org/abs/2401.12926}.

\bibitem[Huang et~al.(2024)Huang, Zhang, Guo, Shang, and Fu]{huang2024dynimpt}
Wei Huang, Yunxiao Zhang, Shangmin Guo, Yuming Shang, and Xiangling Fu.
\newblock Dynimpt: A dynamic data selection method for improving model training efficiency.
\newblock \emph{IEEE Transactions on Knowledge and Data Engineering}, 2024.

\bibitem[Touvron et~al.(2023)Touvron, Martin, Stone, Albert, Almahairi, Babaei, Bashlykov, Batra, Bhargava, Bhosale, et~al.]{llama2}
Hugo Touvron, Louis Martin, Kevin Stone, Peter Albert, Amjad Almahairi, Yasmine Babaei, Nikolay Bashlykov, Soumya Batra, Prajjwal Bhargava, Shruti Bhosale, et~al.
\newblock Llama 2: Open foundation and fine-tuned chat models.
\newblock \emph{arXiv preprint arXiv:2307.09288}, 2023.

\bibitem[Ivison et~al.(2023)Ivison, Wang, Pyatkin, Lambert, Peters, Dasigi, Jang, Wadden, Smith, Beltagy, et~al.]{tulu2}
Hamish Ivison, Yizhong Wang, Valentina Pyatkin, Nathan Lambert, Matthew Peters, Pradeep Dasigi, Joel Jang, David Wadden, Noah~A Smith, Iz~Beltagy, et~al.
\newblock Camels in a changing climate: Enhancing lm adaptation with tulu 2.
\newblock \emph{arXiv preprint arXiv:2311.10702}, 2023.

\bibitem[Grattafiori et~al.(2024)Grattafiori, Dubey, Jauhri, Pandey, Kadian, Al-Dahle, Letman, Mathur, Schelten, Vaughan, et~al.]{llama3}
Aaron Grattafiori, Abhimanyu Dubey, Abhinav Jauhri, Abhinav Pandey, Abhishek Kadian, Ahmad Al-Dahle, Aiesha Letman, Akhil Mathur, Alan Schelten, Alex Vaughan, et~al.
\newblock The llama 3 herd of models.
\newblock \emph{arXiv preprint arXiv:2407.21783}, 2024.

\bibitem[Team(2024)]{qwen2.5}
Qwen Team.
\newblock Qwen2.5: A party of foundation models, September 2024.
\newblock URL \url{https://qwenlm.github.io/blog/qwen2.5/}.

\bibitem[Hendrycks et~al.(2021{\natexlab{a}})Hendrycks, Burns, Basart, Zou, Mazeika, Song, and Steinhardt]{mmlu1}
Dan Hendrycks, Collin Burns, Steven Basart, Andy Zou, Mantas Mazeika, Dawn Song, and Jacob Steinhardt.
\newblock Measuring massive multitask language understanding.
\newblock \emph{Proceedings of the International Conference on Learning Representations (ICLR)}, 2021{\natexlab{a}}.

\bibitem[Hendrycks et~al.(2021{\natexlab{b}})Hendrycks, Burns, Basart, Critch, Li, Song, and Steinhardt]{mmlu2}
Dan Hendrycks, Collin Burns, Steven Basart, Andrew Critch, Jerry Li, Dawn Song, and Jacob Steinhardt.
\newblock Aligning ai with shared human values.
\newblock \emph{Proceedings of the International Conference on Learning Representations (ICLR)}, 2021{\natexlab{b}}.

\bibitem[Sorscher et~al.(2022)Sorscher, Geirhos, Shekhar, Ganguli, and Morcos]{sorscher2022beyond}
Ben Sorscher, Robert Geirhos, Shashank Shekhar, Surya Ganguli, and Ari Morcos.
\newblock Beyond neural scaling laws: beating power law scaling via data pruning.
\newblock \emph{Advances in Neural Information Processing Systems}, 35:\penalty0 19523--19536, 2022.

\bibitem[Abbas et~al.(2023)Abbas, Tirumala, Simig, Ganguli, and Morcos]{abbas2023semdedup}
Amro Abbas, Kushal Tirumala, D{\'a}niel Simig, Surya Ganguli, and Ari~S Morcos.
\newblock Semdedup: Data-efficient learning at web-scale through semantic deduplication.
\newblock \emph{arXiv preprint arXiv:2303.09540}, 2023.

\bibitem[Do and Gaspers(2019)]{do2019cross}
Quynh Do and Judith Gaspers.
\newblock Cross-lingual transfer learning with data selection for large-scale spoken language understanding.
\newblock \emph{EMNLP}, 2019.

\bibitem[Xie et~al.(2023{\natexlab{b}})]{xie2023smalltolarge}
Sang~Michael Xie et~al.
\newblock Smalltolarge (s2l): Scalable data selection for fine-tuning large language models by summarizing training loss trajectories of small models.
\newblock \emph{arXiv preprint}, 2023{\natexlab{b}}.

\bibitem[Brown et~al.(2020)Brown, Mann, Ryder, Subbiah, Kaplan, Dhariwal, Neelakantan, Shyam, Voss, and Amodei]{brown2020language}
Tom~B Brown, Benjamin Mann, Nick Ryder, Melanie Subbiah, Jared Kaplan, Prafulla Dhariwal, Arvind Neelakantan, Pranav Shyam, Girish Voss, and Dario Amodei.
\newblock Language models are few-shot learners.
\newblock \emph{arXiv preprint arXiv:2005.14165}, 2020.

\bibitem[Gao et~al.(2020)Gao, Biderman, Black, Golding, Hoppe, Foster, Phang, He, Thite, Nabeshima, et~al.]{gao2020pile}
Leo Gao, Stella Biderman, Sid Black, Laurence Golding, Travis Hoppe, Charles Foster, Jason Phang, Horace He, Anish Thite, Noa Nabeshima, et~al.
\newblock The pile: An 800gb dataset of diverse text for language modeling.
\newblock \emph{arXiv preprint arXiv:2101.00027}, 2020.

\bibitem[Chowdhery et~al.(2023)Chowdhery, Narang, Devlin, Bosma, Mishra, Roberts, Barham, Chung, Sutton, Gehrmann, et~al.]{chowdhery2023palm}
Aakanksha Chowdhery, Sharan Narang, Jacob Devlin, Maarten Bosma, Gaurav Mishra, Adam Roberts, Paul Barham, Hyung~Won Chung, Charles Sutton, Sebastian Gehrmann, et~al.
\newblock Palm: Scaling language modeling with pathways.
\newblock \emph{Journal of Machine Learning Research}, 24\penalty0 (240):\penalty0 1--113, 2023.

\bibitem[Du et~al.(2022)Du, Huang, Dai, Tong, Lepikhin, Xu, Krikun, Zhou, Yu, Firat, et~al.]{du2022glam}
Nan Du, Yanping Huang, Andrew~M Dai, Simon Tong, Dmitry Lepikhin, Yuanzhong Xu, Maxim Krikun, Yanqi Zhou, Adams~Wei Yu, Orhan Firat, et~al.
\newblock Glam: Efficient scaling of language models with mixture-of-experts.
\newblock In \emph{International Conference on Machine Learning}, pages 5547--5569. PMLR, 2022.

\bibitem[Moore and Lewis(2010)]{moore2010intelligent}
Robert~C Moore and William Lewis.
\newblock Intelligent selection of language model training data.
\newblock In \emph{Proceedings of the ACL 2010 conference short papers}, pages 220--224, 2010.

\bibitem[Axelrod(2017)]{axelrod2017cynical}
Amittai Axelrod.
\newblock Cynical selection of language model training data.
\newblock \emph{arXiv preprint arXiv:1709.02279}, 2017.

\bibitem[Feng et~al.(2022)Feng, Xia, Van~Durme, and Sedoc]{feng2022automatic}
Yukun Feng, Patrick Xia, Benjamin Van~Durme, and Jo{\~a}o Sedoc.
\newblock Automatic document selection for efficient encoder pretraining.
\newblock \emph{arXiv preprint arXiv:2210.10951}, 2022.

\bibitem[Raffel et~al.(2020)Raffel, Shazeer, Roberts, Lee, Narang, Matena, Zhou, Li, and Liu]{raffel2020exploring}
Colin Raffel, Noam Shazeer, Adam Roberts, Katherine Lee, Sharan Narang, Michael Matena, Yanqi Zhou, Wei Li, and Peter~J Liu.
\newblock Exploring the limits of transfer learning with a unified text-to-text transformer.
\newblock \emph{Journal of machine learning research}, 21\penalty0 (140):\penalty0 1--67, 2020.

\bibitem[Chen et~al.(2024)Chen, Roberts, Bhatia, Wang, Zhang, Sala, and R{\'e}]{chen2024skill}
Mayee Chen, Nicholas Roberts, Kush Bhatia, Jue Wang, Ce~Zhang, Frederic Sala, and Christopher R{\'e}.
\newblock Skill-it! a data-driven skills framework for understanding and training language models.
\newblock \emph{Advances in Neural Information Processing Systems}, 36, 2024.

\bibitem[Albalak et~al.(2023)Albalak, Pan, Raffel, and Wang]{albalak2023efficient}
Alon Albalak, Liangming Pan, Colin Raffel, and William~Yang Wang.
\newblock Efficient online data mixing for language model pre-training.
\newblock In \emph{R0-FoMo: Robustness of Few-shot and Zero-shot Learning in Large Foundation Models}, 2023.

\bibitem[Zhou et~al.(2023{\natexlab{a}})]{zhou2023lobass}
Haotian Zhou et~al.
\newblock Lobass: Gauging learnability in supervised fine-tuning data.
\newblock \emph{arXiv preprint arXiv:2310.13008}, 2023{\natexlab{a}}.

\bibitem[Swayamdipta et~al.(2020)Swayamdipta, Schwartz, Lourie, Wang, Hajishirzi, Smith, and Choi]{swayamdipta2020dataset}
Swabha Swayamdipta, Roy Schwartz, Nicholas Lourie, Yizhong Wang, Hannaneh Hajishirzi, Noah~A Smith, and Yejin Choi.
\newblock Dataset cartography: Mapping and diagnosing datasets with training dynamics.
\newblock \emph{EMNLP}, 2020.

\bibitem[Bhatt et~al.(2024)]{bhatt2024experimental}
Gantavya Bhatt et~al.
\newblock An experimental design framework for label-efficient supervised finetuning of large language models.
\newblock \emph{arXiv preprint arXiv:2401.06692}, 2024.

\bibitem[Tirumala et~al.(2023)Tirumala, Simig, Aghajanyan, and Morcos]{DBLP:conf/nips/TirumalaSAM23}
Kushal Tirumala, Daniel Simig, Armen Aghajanyan, and Ari Morcos.
\newblock {D4:} improving {LLM} pretraining via document de-duplication and diversification.
\newblock In Alice Oh, Tristan Naumann, Amir Globerson, Kate Saenko, Moritz Hardt, and Sergey Levine, editors, \emph{Advances in Neural Information Processing Systems 36: Annual Conference on Neural Information Processing Systems 2023, NeurIPS 2023, New Orleans, LA, USA, December 10 - 16, 2023}, 2023.
\newblock URL \url{http://papers.nips.cc/paper\_files/paper/2023/hash/a8f8cbd7f7a5fb2c837e578c75e5b615-Abstract-Datasets\_and\_Benchmarks.html}.

\bibitem[Engstrom et~al.(2024)Engstrom, Feldmann, and Madry]{DBLP:conf/icml/EngstromFM24}
Logan Engstrom, Axel Feldmann, and Aleksander Madry.
\newblock Dsdm: Model-aware dataset selection with datamodels.
\newblock In \emph{Forty-first International Conference on Machine Learning, {ICML} 2024, Vienna, Austria, July 21-27, 2024}. OpenReview.net, 2024.
\newblock URL \url{https://openreview.net/forum?id=GC8HkKeH8s}.

\bibitem[Kaushal et~al.(2019)Kaushal, Iyer, Kothawade, Mahadev, Doctor, and Ramakrishnan]{kaushal2019learning}
Vishal Kaushal, Rishabh Iyer, Suraj Kothawade, Rohan Mahadev, Khoshrav Doctor, and Ganesh Ramakrishnan.
\newblock Learning from less data: A unified data subset selection and active learning framework for computer vision.
\newblock In \emph{2019 IEEE Winter Conference on Applications of Computer Vision (WACV)}, pages 1289--1299. IEEE, 2019.

\bibitem[Killamsetty et~al.(2021)Killamsetty, Zhao, Chen, and Iyer]{killamsetty2021retrieve}
Krishnateja Killamsetty, Xujiang Zhao, Feng Chen, and Rishabh Iyer.
\newblock Retrieve: Coreset selection for efficient and robust semi-supervised learning.
\newblock \emph{Advances in neural information processing systems}, 34:\penalty0 14488--14501, 2021.

\bibitem[Wei et~al.(2015)Wei, Iyer, and Bilmes]{wei2015submodularity}
Kai Wei, Rishabh Iyer, and Jeff Bilmes.
\newblock Submodularity in data subset selection and active learning.
\newblock In \emph{International conference on machine learning}, pages 1954--1963. PMLR, 2015.

\bibitem[Chen et~al.(2023)Chen, Li, Yan, Wang, Gunaratna, Yadav, Tang, Srinivasan, Zhou, Huang, et~al.]{chen2023alpagasus}
Lichang Chen, Shiyang Li, Jun Yan, Hai Wang, Kalpa Gunaratna, Vikas Yadav, Zheng Tang, Vijay Srinivasan, Tianyi Zhou, Heng Huang, et~al.
\newblock Alpagasus: Training a better alpaca with fewer data.
\newblock \emph{arXiv preprint arXiv:2307.08701}, 2023.

\bibitem[Cao et~al.(2023)Cao, Kang, Wang, and Sun]{cao2023instruction}
Yihan Cao, Yanbin Kang, Chi Wang, and Lichao Sun.
\newblock Instruction mining: Instruction data selection for tuning large language models.
\newblock \emph{arXiv preprint arXiv:2307.06290}, 2023.

\bibitem[Sener and Savarese(2017)]{sener2017active}
Ozan Sener and Silvio Savarese.
\newblock Active learning for convolutional neural networks: A core-set approach.
\newblock \emph{arXiv preprint arXiv:1708.00489}, 2017.

\bibitem[Krizhevsky(2009)]{cifar10}
Alex Krizhevsky.
\newblock Learning multiple layers of features from tiny images.
\newblock Technical report, 2009.

\bibitem[Kraft(1988)]{slsqp}
Dieter Kraft.
\newblock A software package for sequential quadratic programming.
\newblock \emph{Forschungsbericht- Deutsche Forschungs- und Versuchsanstalt fur Luft- und Raumfahrt}, 1988.

\bibitem[Virtanen et~al.(2020)Virtanen, Gommers, Oliphant, Haberland, Reddy, Cournapeau, Burovski, Peterson, Weckesser, Bright, et~al.]{scipy}
P~Virtanen, R~Gommers, TE~Oliphant, M~Haberland, T~Reddy, D~Cournapeau, E~Burovski, P~Peterson, W~Weckesser, J~Bright, et~al.
\newblock Fundamental algorithms for scientific computing in python and scipy 1.0 contributors. scipy 1.0.
\newblock \emph{Nat. Methods}, 17:\penalty0 261--272, 2020.

\bibitem[Kingma(2014)]{adam}
Diederik~P Kingma.
\newblock Adam: A method for stochastic optimization.
\newblock \emph{arXiv preprint arXiv:1412.6980}, 2014.

\bibitem[Cobbe et~al.(2021)Cobbe, Kosaraju, Bavarian, Chen, Jun, Kaiser, Plappert, Tworek, Hilton, Nakano, Hesse, and Schulman]{gsm8k}
Karl Cobbe, Vineet Kosaraju, Mohammad Bavarian, Mark Chen, Heewoo Jun, Lukasz Kaiser, Matthias Plappert, Jerry Tworek, Jacob Hilton, Reiichiro Nakano, Christopher Hesse, and John Schulman.
\newblock Training verifiers to solve math word problems.
\newblock \emph{arXiv preprint arXiv:2110.14168}, 2021.

\bibitem[Suzgun et~al.(2022)Suzgun, Scales, Sch{\"a}rli, Gehrmann, Tay, Chung, Chowdhery, Le, Chi, Zhou, , and Wei]{bbh}
Mirac Suzgun, Nathan Scales, Nathanael Sch{\"a}rli, Sebastian Gehrmann, Yi~Tay, Hyung~Won Chung, Aakanksha Chowdhery, Quoc~V Le, Ed~H Chi, Denny Zhou, , and Jason Wei.
\newblock Challenging big-bench tasks and whether chain-of-thought can solve them.
\newblock \emph{arXiv preprint arXiv:2210.09261}, 2022.

\bibitem[Clark et~al.(2020)Clark, Choi, Collins, Garrette, Kwiatkowski, Nikolaev, and Palomaki]{tydiqa}
Jonathan~H. Clark, Eunsol Choi, Michael Collins, Dan Garrette, Tom Kwiatkowski, Vitaly Nikolaev, and Jennimaria Palomaki.
\newblock Tydi qa: A benchmark for information-seeking question answering in typologically diverse languages.
\newblock \emph{Transactions of the Association for Computational Linguistics}, 2020.

\bibitem[Chen et~al.(2021)Chen, Tworek, Jun, Yuan, Pinto, Kaplan, Edwards, Burda, Joseph, Brockman, et~al.]{codex}
Mark Chen, Jerry Tworek, Heewoo Jun, Qiming Yuan, Henrique Ponde De~Oliveira Pinto, Jared Kaplan, Harri Edwards, Yuri Burda, Nicholas Joseph, Greg Brockman, et~al.
\newblock Evaluating large language models trained on code.
\newblock \emph{arXiv preprint arXiv:2107.03374}, 2021.

\bibitem[Rajpurkar et~al.(2016)Rajpurkar, Zhang, Lopyrev, and Liang]{squad}
Pranav Rajpurkar, Jian Zhang, Konstantin Lopyrev, and Percy Liang.
\newblock {SQ}u{AD}: 100,000+ questions for machine comprehension of text.
\newblock In Jian Su, Kevin Duh, and Xavier Carreras, editors, \emph{Proceedings of the 2016 Conference on Empirical Methods in Natural Language Processing}, pages 2383--2392, Austin, Texas, November 2016. Association for Computational Linguistics.
\newblock \doi{10.18653/v1/D16-1264}.
\newblock URL \url{https://aclanthology.org/D16-1264}.

\bibitem[Li et~al.(2023{\natexlab{b}})Li, Zhang, Dubois, Taori, Gulrajani, Guestrin, Liang, and Hashimoto]{li2023alpacaeval}
Xuechen Li, Tianyi Zhang, Yann Dubois, Rohan Taori, Ishaan Gulrajani, Carlos Guestrin, Percy Liang, and Tatsunori~B Hashimoto.
\newblock Alpacaeval: An automatic evaluator of instruction-following models, 2023{\natexlab{b}}.

\bibitem[Kaplan et~al.(2020)Kaplan, McCandlish, Henighan, Brown, Chess, Child, Gray, Radford, Wu, and Amodei]{kaplan2020scaling}
Jared Kaplan, Sam McCandlish, Tom Henighan, Tom~B Brown, Benjamin Chess, Rewon Child, Scott Gray, Alec Radford, Jeffrey Wu, and Dario Amodei.
\newblock Scaling laws for neural language models.
\newblock \emph{arXiv preprint arXiv:2001.08361}, 2020.

\bibitem[Cobb et~al.(2024)Cobb, Baydin, Pearlmutter, and Jha]{jvp1}
Adam~D Cobb, At{\i}l{\i}m~G{\"u}ne{\c{s}} Baydin, Barak~A Pearlmutter, and Susmit Jha.
\newblock Second-order forward-mode automatic differentiation for optimization.
\newblock \emph{arXiv preprint arXiv:2408.10419}, 2024.

\bibitem[Hu et~al.(2022)Hu, Shen, Wallis, Allen-Zhu, Li, Wang, Wang, Chen, et~al.]{hu2022lora}
Edward~J Hu, Yelong Shen, Phillip Wallis, Zeyuan Allen-Zhu, Yuanzhi Li, Shean Wang, Lu~Wang, Weizhu Chen, et~al.
\newblock Lora: Low-rank adaptation of large language models.
\newblock \emph{ICLR}, 1\penalty0 (2):\penalty0 3, 2022.

\bibitem[Park et~al.(2023)Park, Georgiev, Ilyas, Leclerc, and Madry]{trak}
Sung~Min Park, Kristian Georgiev, Andrew Ilyas, Guillaume Leclerc, and Aleksander Madry.
\newblock Trak: Attributing model behavior at scale.
\newblock \emph{arXiv preprint arXiv:2303.14186}, 2023.

\bibitem[Foret et~al.(2020)Foret, Kleiner, Mobahi, and Neyshabur]{sam}
Pierre Foret, Ariel Kleiner, Hossein Mobahi, and Behnam Neyshabur.
\newblock Sharpness-aware minimization for efficiently improving generalization.
\newblock \emph{arXiv preprint arXiv:2010.01412}, 2020.

\bibitem[Peste et~al.(2022)Peste, Vladu, Kurtic, Lampert, and Alistarh]{cram}
Alexandra Peste, Adrian Vladu, Eldar Kurtic, Christoph~H Lampert, and Dan Alistarh.
\newblock Cram: A compression-aware minimizer.
\newblock \emph{arXiv preprint arXiv:2207.14200}, 2022.

\bibitem[Chung et~al.(2024)Chung, Hou, Longpre, Zoph, Tay, Fedus, Li, Wang, Dehghani, Brahma, et~al.]{flan}
Hyung~Won Chung, Le~Hou, Shayne Longpre, Barret Zoph, Yi~Tay, William Fedus, Yunxuan Li, Xuezhi Wang, Mostafa Dehghani, Siddhartha Brahma, et~al.
\newblock Scaling instruction-finetuned language models.
\newblock \emph{Journal of Machine Learning Research}, 25\penalty0 (70):\penalty0 1--53, 2024.

\bibitem[K{\"o}pf et~al.(2023)K{\"o}pf, Kilcher, Von~R{\"u}tte, Anagnostidis, Tam, Stevens, Barhoum, Nguyen, Stanley, Nagyfi, et~al.]{openassistant}
Andreas K{\"o}pf, Yannic Kilcher, Dimitri Von~R{\"u}tte, Sotiris Anagnostidis, Zhi~Rui Tam, Keith Stevens, Abdullah Barhoum, Duc Nguyen, Oliver Stanley, Rich{\'a}rd Nagyfi, et~al.
\newblock Openassistant conversations-democratizing large language model alignment.
\newblock \emph{Advances in Neural Information Processing Systems}, 36:\penalty0 47669--47681, 2023.

\bibitem[Conover et~al.(2023)Conover, Hayes, Mathur, Xie, Wan, Shah, Ghodsi, Wendell, Zaharia, and Xin]{dolly}
Mike Conover, Matt Hayes, Ankit Mathur, Jianwei Xie, Jun Wan, Sam Shah, Ali Ghodsi, Patrick Wendell, Matei Zaharia, and Reynold Xin.
\newblock Free dolly: Introducing the world’s first truly open instruction-tuned llm, 2023.

\bibitem[Peng et~al.(2023)Peng, Li, He, Galley, and Gao]{gpt4-alpaca}
Baolin Peng, Chunyuan Li, Pengcheng He, Michel Galley, and Jianfeng Gao.
\newblock Instruction tuning with gpt-4.
\newblock \emph{arXiv preprint arXiv:2304.03277}, 2023.

\bibitem[Chaudhary(2023)]{codealpaca}
Sahil Chaudhary.
\newblock Code alpaca: An instruction-following llama model for code generation, 2023.

\bibitem[Zhou et~al.(2023{\natexlab{b}})Zhou, Liu, Xu, Iyer, Sun, Mao, Ma, Efrat, Yu, Yu, et~al.]{lima}
Chunting Zhou, Pengfei Liu, Puxin Xu, Srinivasan Iyer, Jiao Sun, Yuning Mao, Xuezhe Ma, Avia Efrat, Ping Yu, Lili Yu, et~al.
\newblock Lima: Less is more for alignment.
\newblock \emph{Advances in Neural Information Processing Systems}, 36:\penalty0 55006--55021, 2023{\natexlab{b}}.

\bibitem[Xu et~al.(2023)Xu, Sun, Zheng, Geng, Zhao, Feng, Tao, and Jiang]{wizardlm}
Can Xu, Qingfeng Sun, Kai Zheng, Xiubo Geng, Pu~Zhao, Jiazhan Feng, Chongyang Tao, and Daxin Jiang.
\newblock Wizardlm: Empowering large language models to follow complex instructions.
\newblock \emph{arXiv preprint arXiv:2304.12244}, 2023.

\bibitem[Lian et~al.(2023)Lian, Goodson, Pentland, Cook, Vong, and “Teknium”]{openorca}
Wing Lian, Bleys Goodson, Eugene Pentland, Austin Cook, Chanvichet Vong, and “Teknium”.
\newblock Openorca: An open dataset of gpt augmented flan reasoning traces, 2023.

\bibitem[Wadden et~al.(2024)Wadden, Shi, Morrison, Naik, Singh, Barzilay, Lo, Hope, Soldaini, Shen, et~al.]{sciriff}
David Wadden, Kejian Shi, Jacob Morrison, Aakanksha Naik, Shruti Singh, Nitzan Barzilay, Kyle Lo, Tom Hope, Luca Soldaini, Shannon~Zejiang Shen, et~al.
\newblock Sciriff: A resource to enhance language model instruction-following over scientific literature.
\newblock \emph{arXiv preprint arXiv:2406.07835}, 2024.

\bibitem[Anil et~al.(2023)Anil, Dai, Firat, Johnson, Lepikhin, Passos, Shakeri, Taropa, Bailey, Chen, et~al.]{anil2023palm}
Rohan Anil, Andrew~M Dai, Orhan Firat, Melvin Johnson, Dmitry Lepikhin, Alexandre Passos, Siamak Shakeri, Emanuel Taropa, Paige Bailey, Zhifeng Chen, et~al.
\newblock Palm 2 technical report.
\newblock \emph{arXiv preprint arXiv:2305.10403}, 2023.

\bibitem[Lee et~al.(2024)Lee, Roy, Xu, Raiman, Shoeybi, Catanzaro, and Ping]{lee2024nv}
Chankyu Lee, Rajarshi Roy, Mengyao Xu, Jonathan Raiman, Mohammad Shoeybi, Bryan Catanzaro, and Wei Ping.
\newblock Nv-embed: Improved techniques for training llms as generalist embedding models.
\newblock \emph{arXiv preprint arXiv:2405.17428}, 2024.

\bibitem[Ni et~al.(2021)Ni, Qu, Lu, Dai, {\'A}brego, Ma, Zhao, Luan, Hall, Chang, et~al.]{gtr}
Jianmo Ni, Chen Qu, Jing Lu, Zhuyun Dai, Gustavo~Hern{\'a}ndez {\'A}brego, Ji~Ma, Vincent~Y Zhao, Yi~Luan, Keith~B Hall, Ming-Wei Chang, et~al.
\newblock Large dual encoders are generalizable retrievers.
\newblock \emph{arXiv preprint arXiv:2112.07899}, 2021.

\bibitem[Zhao et~al.(2024)Zhao, Zhang, Chen, Wang, Anandkumar, and Tian]{zhao2024galore}
Jiawei Zhao, Zhenyu Zhang, Beidi Chen, Zhangyang Wang, Anima Anandkumar, and Yuandong Tian.
\newblock Galore: Memory-efficient llm training by gradient low-rank projection.
\newblock \emph{arXiv preprint arXiv:2403.03507}, 2024.

\bibitem[Agarwal et~al.(2024)Agarwal, Astra, Hoque, Srivatsa, Ganti, Wright, and Chen]{hadacore}
Krish Agarwal, Rishi Astra, Adnan Hoque, Mudhakar Srivatsa, Raghu Ganti, Less Wright, and Sijia Chen.
\newblock Hadacore: Tensor core accelerated hadamard transform kernel.
\newblock \emph{arXiv preprint arXiv:2412.08832}, 2024.

\bibitem[Dao(2023)]{tridao}
Tri Dao.
\newblock Fast hadamard transform in cuda, with a pytorch interface, 2023.
\newblock URL \url{https://github.com/Dao-AILab/fast-hadamard-transform}.

\end{thebibliography}

%%%%%%%%%%%%%%%%%%%%%%%%%%%%%%%%%%%%%%%%%%%%%%%%%%%%%%%%%%%%%%%%%%%%%%%%%%%%%%%
%%%%%%%%%%%%%%%%%%%%%%%%%%%%%%%%%%%%%%%%%%%%%%%%%%%%%%%%%%%%%%%%%%%%%%%%%%%%%%%
% APPENDIX
%%%%%%%%%%%%%%%%%%%%%%%%%%%%%%%%%%%%%%%%%%%%%%%%%%%%%%%%%%%%%%%%%%%%%%%%%%%%%%%
%%%%%%%%%%%%%%%%%%%%%%%%%%%%%%%%%%%%%%%%%%%%%%%%%%%%%%%%%%%%%%%%%%%%%%%%%%%%%%%
\newpage
\appendix

\section{Gradient Analysis During Training}
\label{apx:grad-analysis}

In this section, we analyze the behavior of gradients throughout training. We fine-tune a LLaMA2-7B model \citep{llama2} on 10000 randomly selected samples from Tulu V2 \citep{tulu2} for 2 epochs, saving model checkpoints every 10 steps.

For each checkpoint—including the initial and final models—we compute the gradients of 1000 held-out samples from Tulu V2, as well as samples from the target dataset BBH \citep{bbh}, and project them into an 8192-dimensional space using random Rademacher matrices, following the efficient GPU implementation of \citet{trak}, also adopted in \citet{less}. For each dataset, we compute the average gradient cosine similarity across checkpoints. As shown in Figure~\ref{fig:grad-sim-tulu-bbh}, while the gradient directions can change substantially in the early steps, they stabilize quickly during training. This observation justifies the use of a short warm-up phase as both necessary and sufficient. Similar plots for GSM8k \citep{gsm8k} and SQuAD \citep{squad} are provided later in the Appendix (Figure~\ref{fig:grad-sim-other}).

Additionally, for each dataset and checkpoint, we measure the Pearson product-moment correlation between gradient norms and the number of label tokens per sample. As shown in Figure~\ref{fig:grad-len-corr}, we observe a consistent negative correlation, which supports our decision to normalize gradients prior to distillation.

\section{Linear Model Study}

\begin{figure}[t]
    \centering
    \includegraphics[width=\linewidth]{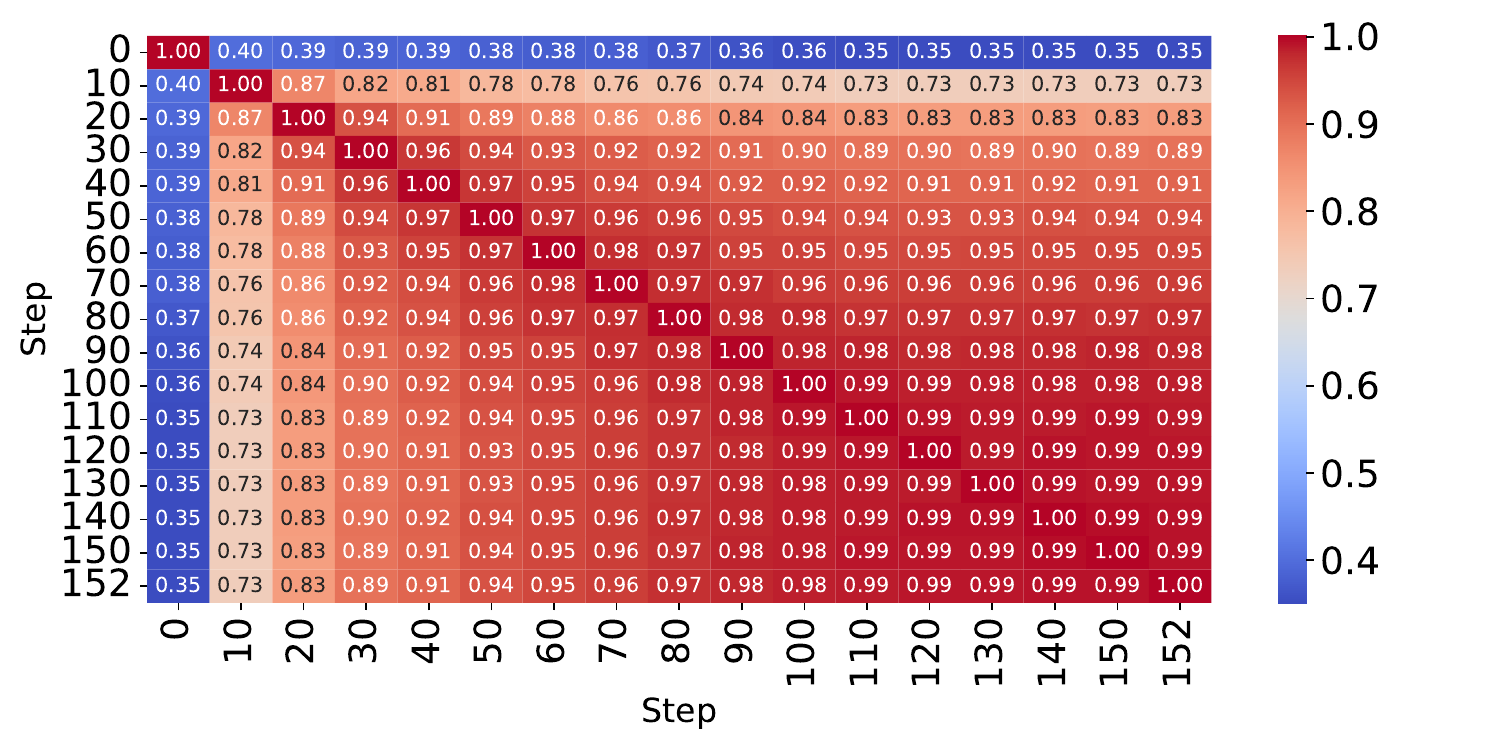}
    \includegraphics[width=\linewidth]{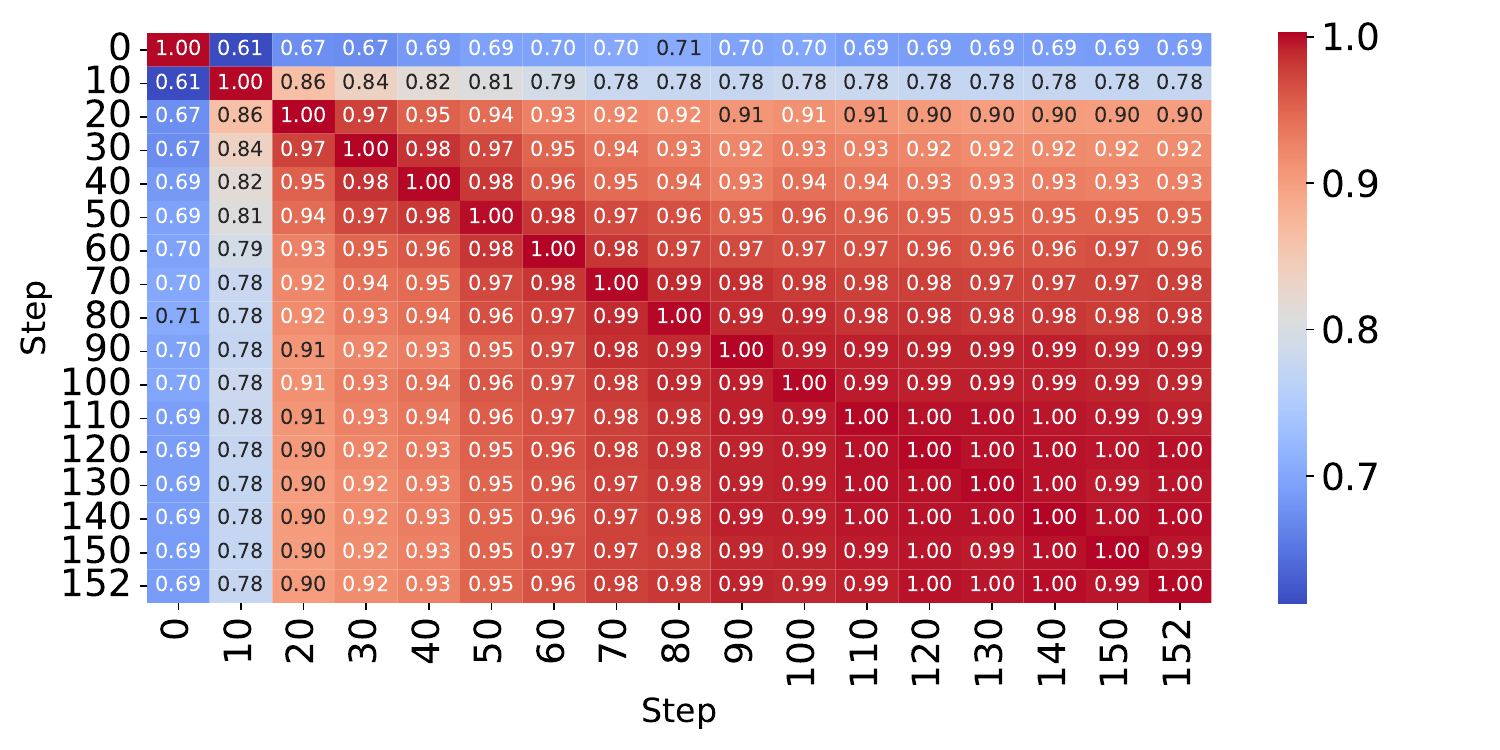}
    \caption{Average gradient cosine similarity on unseen samples from Tulu V2 (top) and BBH (bottom) across checkpoints.}
    \label{fig:grad-sim-tulu-bbh}
\end{figure}

In this section, we show that a regularization term can be effective in robustifying Objective \ref{eq:local-obj} to small changes in the model parameters $\vtheta$, when the model is linear and the loss is quadratic.

\label{apx:linear}
For a fixed $\epsilon > 0$, define a new objective as below:
\begin{equation}
\vw^* = \argmin_{\vw}~ \max_{||\vdelta|| \le \epsilon}~ f(\vw;\vtheta + \vdelta),
\label{eq:robust-obj-lin}
\end{equation}
minimizing the maximum value of $f$ around a point $\vtheta$ in a neighborhood of radius $\epsilon$.
This ensures the weights are stable as long as $\vtheta$ is in this neighborhood.
To solve for $w$, we employ Lemma \ref{lem:1} below:

\begin{lemma}
\label{lem:1}
Assume $\loss(\vtheta; \dD, \vw) = \sum_{i=1}^{|\dD|} \vw_i (\langle \vx^{\dD}_i, \vtheta \rangle - y^{\dD}_i)^2$, where $\dD=\{(\vx^{\dD}_1, y^{\dD}_1), (\vx^{\dD}_2, y^{\dD}_2), ..., (\vx^{\dD}_{|\dD|}, y^{\dD}_{|\dD|})\}$ is a dataset. For datasets $\dS$ and $\dT$, let $\mH_{\dT} = \nabla^2_{\vtheta} \loss(\vtheta;\dT,\vone)$, $\vg_\dT = \nabla_{\vtheta} \loss(\vtheta;\dT,\vone)$, $\mH_\vw = \nabla^2_{\vtheta} \loss(\vtheta;\dS,\vw)$, and $\vg_\vw = \nabla_{\vtheta} \loss(\vtheta;\dS,\vw)$. Define $\va_\vw$ and $\mB_\vw$ as below:
\begin{align*}
    \va_\vw &= - \mH_\vw \vg_\dT - \mH_\dT \vg_\vw + \eta \mH_\vw \mH_\dT \vg_\vw \numberthis
\end{align*}
\begin{align*}
    \mB_\vw &= - \mH_\dT \mH_\vw + \frac{\eta}{2} \mH_\vw \mH_\dT \mH_\vw \numberthis
\end{align*}

In the setting above (linear model with quadratic loss), the function $f$ has the property that $\forall~ \vtheta,\vdelta \in \Roned{d}, \vw \in \Roned{n}$:
\begin{equation}
    f(\vw;\vtheta + \vdelta) = f(\vw;\vtheta) + \va_\vw^T\vdelta + \vdelta^T \mB_\vw \vdelta.
\end{equation}
\end{lemma}
\begin{proof}
First notice that, for simplicity, the loss here is defined as the sum (as opposed to the average) of per-sample losses, which drops the $\frac{1}{|\dS|}$ terms in the loss, gradient, Hessian, and $\mQ$ objects. 
Recalling the definition of $f$ from \ref{eq:f}, we can write $f(\vw;\vtheta+\vdelta) = -\vp(\vtheta+\vdelta)^T \vw + \frac{\eta}{2} \vw^T \mQ(\vtheta+\vdelta)\vw$.
Since the loss is quadratic in $\vtheta$, the Hessian is independent of $\vtheta$, and the derivatives above the second order are zero. Hence, defining $\vg_i^{\dD}(\vtheta)$ and $\mH_i^{\dD}$ as the gradient and Hessian of the sample $i$ in $\dD$, we can write:
\begin{equation}
    \vg^{\dD}_i(\vtheta+\vdelta) = \vg^{\dD}_i(\vtheta) + \mH_i^{\dD} \vdelta
\end{equation}
for any $\vdelta$ with the same dimension as $\vtheta$. Setting $\dD = \dT$ and summing across samples, we can write:
\begin{equation}
    \vg_{\dT}(\vtheta+\vdelta) = \vg_{\dT}(\vtheta) + \mH_{\dT} \vdelta
\end{equation}
Additionally, setting $\dD=\dS$ and taking a weighted sum we can write:
\begin{equation}
    \mG_{\dS}(\vtheta+\vdelta) \vw = \mG_{\dS}(\vtheta) \vw + \mH_{\vw} \vdelta
\end{equation}
Next, we see that,
\begin{align*}
    \vp(\vtheta+\vdelta)^T \vw &= \vg_\dT(\vtheta+\vdelta) \mG_\dS(\vtheta+\vdelta)\vw \\
    &= (\vg_{\dT}(\vtheta)^T + \vdelta^T \mH_{\dT})(\mG_{\dS}(\vtheta) \vw + \mH_{\vw} \vdelta) \\
    &= \vp(\vtheta)^T \vw + (\vg_\dT(\vtheta)^T \mH_\vw + \vg_\vw(\vtheta)^T \mH_\dT) \vdelta + \vdelta^T \mH_\dT \mH_\vw \vdelta \numberthis
\end{align*}
And,
\begin{align*}
    \vw^T \mQ(\vtheta+\vdelta)^T \vw &= \vw^T \mG_\dS(\vtheta+\vdelta)^T \mH_\dT \mG_\dS(\vtheta+\vdelta) \vw \\
    &= (\vw^T \mG_\dS(\vtheta)^T + \vdelta^T \mH_\vw) \mH_\dT (\mG_\dS(\vtheta)\vw + \mH_\vw \vdelta) \\
    &= \vw^T \mQ(\vtheta) \vw + 2\vg_\vw(\vtheta)^T \mH_\dT \mH_\vw \vdelta + \vdelta^T \mH_\vw \mH_\dT \mH_\vw \vdelta \numberthis
\end{align*}
Putting them together:
\begin{align*}
    f(\vw;\vtheta+\vdelta) &= -\vp(\vtheta+\vdelta)^T \vw + \frac{\eta}{2} \vw^T \mQ(\vtheta+\vdelta)\vw \\
    &= f(\vw;\vtheta) - ((\vg_\dT(\vtheta)^T \mH_\vw + \vg_\vw(\vtheta)^T \mH_\dT) \vdelta + \vdelta^T \mH_\dT \mH_\vw \vdelta) \\
    & \quad\quad\quad\quad+ \frac{\eta}{2}(2\vg_\vw(\vtheta)^T \mH_\dT \mH_\vw \vdelta + \vdelta^T \mH_\vw \mH_\dT \mH_\vw \vdelta) \\
    &= f(\vw;\vtheta) + (-\vg_\dT(\vtheta)^T \mH_\vw - \vg_\vw(\vtheta)^T \mH_\dT + \eta \vg_\vw(\vtheta)^T \mH_\dT \mH_\vw) \vdelta \\
    & \quad\quad\quad\quad+ \vdelta^T (-\mH_\dT \mH_\vw + \frac{\eta}{2} \mH_\vw \mH_\dT \mH_\vw) \vdelta \\
    &= f(\vw;\vtheta) + \va_\vw^T\vdelta + \vdelta^T \mB_\vw \vdelta
\end{align*}
which concludes the proof.
\end{proof}

Substituting the result of the \ref{lem:1} into Objective \ref{eq:robust-obj-lin}, we can write
\begin{align*}
    \vw^* &= \argmin_{\vw}~ \max_{||\vdelta|| \le \epsilon}~ \Big[ f(\vw;\vtheta) + \va_\vw^T\vdelta + \vdelta^T \mB_\vw \vdelta \Big] \\
    &= \argmin_{\vw}~ \Big[ f(\vw;\vtheta) + \max_{||\vdelta|| \le \epsilon}~ (\va_\vw^T\vdelta + \vdelta^T \mB_\vw \vdelta) \Big] \numberthis \label{eq:minmax}
\end{align*}
We maximize $r(\vdelta) = \va_\vw^T\vdelta + \vdelta^T \mB_\vw \vdelta$ in the sphere with radius $\epsilon$
approximately by taking a single step of size $\epsilon$ in the gradient direction, i.e., $\vdelta^* \approx \epsilon \cdot \frac{r'(\vzero)}{||r'(\vzero)||}$.
This approximation is standard in the sharpness-aware optimization literature \citep{sam, cram}, which addresses a similar min-max objective to search for flat minima.
Note that $r'(\vdelta) = \va_\vw + (\mB_\vw + \mB_\vw^T) \vdelta$, hence $r'(\vzero) = \va_\vw$ and
\begin{equation}
    \max_{||\vdelta|| \le \epsilon}~ (\va_\vw^T\vdelta + \vdelta^T \mB_\vw \vdelta) \approx \epsilon \cdot ||\va_\vw|| + \epsilon^2 \cdot \frac{\va_\vw^T \mB_\vw \va_\vw}{||\va_\vw||^2}.
\end{equation}
Substituting into Equation \ref{eq:minmax}, we get the following objective:
\begin{equation}
    \vw^* \approx \argmin_{\vw}~ \Big[ f(\vw;\vtheta) + \epsilon \cdot ||\va_\vw|| + \epsilon^2 \cdot \frac{\va_\vw^T \mB_\vw \va_\vw}{||\va_\vw||^2} \Big].
\end{equation}
This suggests that the robustness of the weights can be controlled via the hyperparameter $\epsilon$, which determines the strength of the regularization.

We apply this regularization to the running example introduced in Section \ref{sec:running}. As shown in Figure \ref{fig:robust-linear-eps}, using the tuned value $\epsilon = 10^{-4}$ yields better performance than the default weights. However, due to the high computational cost of this regularization term, we use standard L2 regularization for general non-linear models.

\begin{figure}[t]
    \centering
    \includegraphics[width=0.8\linewidth]{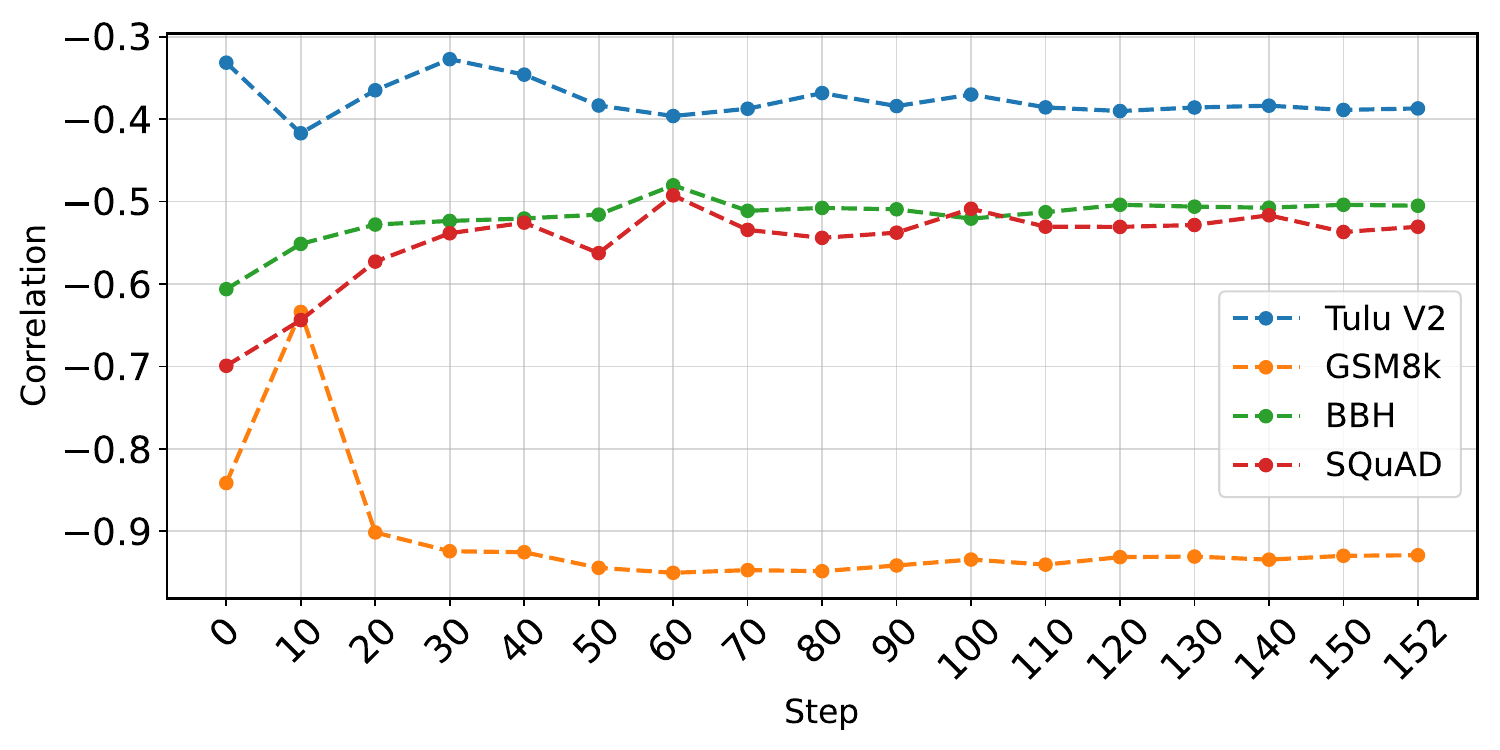}
    \caption{Correlation between gradient norm and number of label tokens, across checkpoints on four datasets.}
    \label{fig:grad-len-corr}
\end{figure}

\section{Adam Optimizer}
\label{apx:adam-greedy}
Here we derive Equations \ref{eq:p-adam} and \ref{eq:Q-adam}, which adapt the vector $\vp$ and the matrix $\mQ$ to the case of the Adam optimizer. Assume that after a warm-up phase, the first- and second-moment estimates of Adam are $\vm$ and $\vv$, respectively. For a new gradient $\vg$, the Adam update rule can be written as:

\begin{equation}
\Delta \vtheta = - \eta \cdot \frac{\vm'}{\sqrt{\vv'} + \epsilon}
\end{equation}

where $\eta$ is the learning rate, and $\vm' = \frac{\beta_1 \vm + (1 - \beta_1)\vg}{1 - \beta_1^s}$ and $\vv' = \frac{\beta_2 \vv + (1 - \beta_2)\vg^2}{1 - \beta_2^s}$ are the updated moment estimates, with $(\beta_1, \beta_2)$ being the Adam beta values for first- and second-order estimate updates, and $s$ being the number steps the optimizer has already been trained for.

For a single update, we note that $\beta_2 \vv + (1 - \beta_2)\vg^2 \approx \vv$. That is because (1) the value $\beta_2$ is typically very close to 1, e.g., 0.995 or 0.999, and (2) due to the warm-up, $\vv$ is stabilized and is not expected to change much. This allows us to ignore the dependence of $\vv'$ on $\vg$, i.e., $\vv' \approx \frac{\vv}{1 - \beta_2^s}$ simplifying the computations. 

Enabled by this, we revisit the Taylor expansion in Equation \ref{eq:taylor}: 
\begin{align*}
    \vw^* &= \argmin_{\vw}~\loss \big(\mech^{\text{Adam}}(\vtheta; \dS, \vw);\dT, \mathbb{1}\big) \\
    &= \argmin_{\vw}~\loss (\vtheta - \frac{\eta}{|\dS|} \frac{\frac{\beta_1 \vm + (1-\beta_1)\mG_\dS(\vtheta) \vw}{1 - \beta_1^s}}{\sqrt{\frac{\vv}{1 - \beta_2^s}} + \epsilon}) \\
    &= \argmin_{\vw}~\loss (\vtheta - \frac{\eta}{|\dS|} [\frac{\beta_1 \vm}{(1-\beta_1^s) (\sqrt{\frac{\vv}{1 - \beta_2^s}} + \epsilon)} + \frac{(1-\beta_1) \mG_\dS(\vtheta) \vw}{(1-\beta_1^s) (\sqrt{\frac{\vv}{1 - \beta_2^s}} + \epsilon)}]) \numberthis \label{eq:taylor-adam-mid}
\end{align*}
Let $\va = \frac{(1-\beta_1)}{(1-\beta_1^s) (\sqrt{\frac{\vv}{1 - \beta_2^s}} + \epsilon)}$ and $\vb = \frac{\beta_1 \vm}{(1-\beta_1^s) (\sqrt{\frac{\vv}{1 - \beta_2^s}} + \epsilon)}$. Construct $\mG^{\text{Adam}}_\dS(\vtheta)$ by element-wise multiplying each column of $\mG_\dS(\vtheta)$ by $\va$. We can now continue Equation \ref{eq:taylor-adam-mid} by:
\begin{align*}
     \vw^* &= \argmin_{\vw}~\loss (\vtheta - \frac{\eta}{|\dS|} (\vb + \mG^{\text{Adam}}_\dS(\vtheta) \vw) \\
     &\approx \argmin_{\vw}~[\loss (\vtheta;\dT, \mathbb{1}) - \frac{\eta}{|\dS|} \vg_\dT^T(\vtheta) (\vb + \mG^{\text{Adam}}_\dS(\vtheta) \vw) + \\ & \quad\quad\quad\frac{\eta^2}{2~|\dS|^2} (\vb^T + \vw^T \mG^{\text{Adam}}_\dS(\vtheta)^T) \mH_\dT(\vtheta) (\vb + \mG^{\text{Adam}}_\dS(\vtheta) \vw) ] \\
    &= \argmin_{\vw}~ -(\vg_\dT^T(\vtheta) - \frac{\eta}{|\dS|} \vb^T \mH_\dT(\vtheta)) \mG^{\text{Adam}}_\dS(\vtheta) \vw \\
    &\quad\quad\quad+ \frac{\eta}{2~|\dS|} \vw^T \mG^{\text{Adam}}_\dS(\vtheta)^T \mH_\dT(\vtheta) \mG^{\text{Adam}}_\dS(\vtheta) \vw \\
    &= \argmin_{\vw}~ -\vp^{\text{Adam}}(\vtheta)^T \vw
    + \frac{\eta}{2} \vw^T \mQ^{\text{Adam}}(\vtheta)^T \vw \numberthis
\end{align*}
Where $\vp^{\text{Adam}}$ and $\mQ^{\text{Adam}}$ are defined in Equations \ref{eq:p-adam} and \ref{eq:Q-adam}, respectively.

\begin{figure}[t] % The * makes it span both columns
    \centering
    \begin{subfigure}{0.47\textwidth}
        \includegraphics[width=\linewidth]{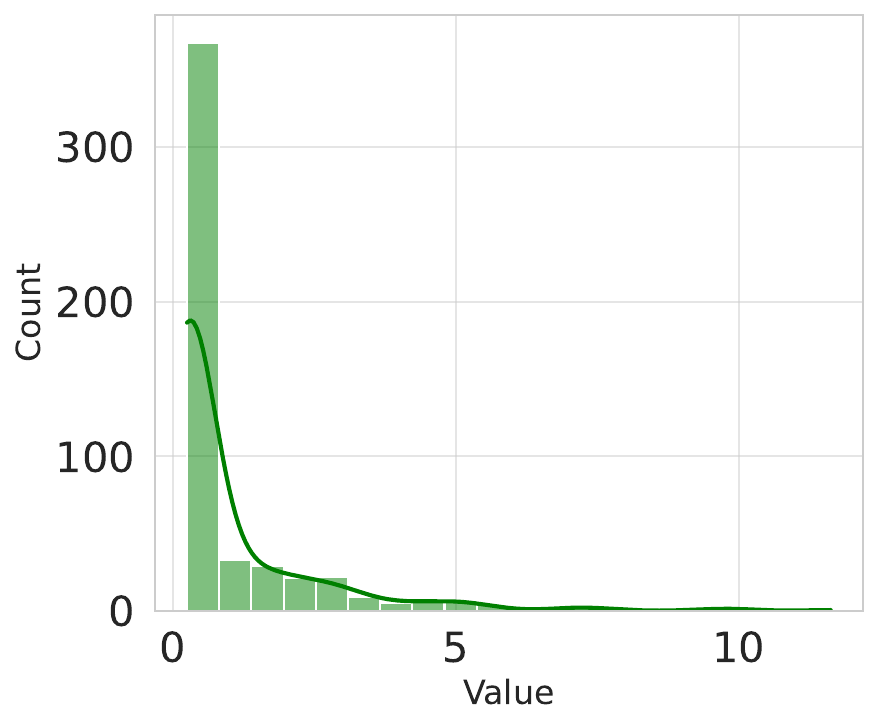}
    \end{subfigure}
    \begin{subfigure}{0.47\textwidth}
        \includegraphics[width=\linewidth]{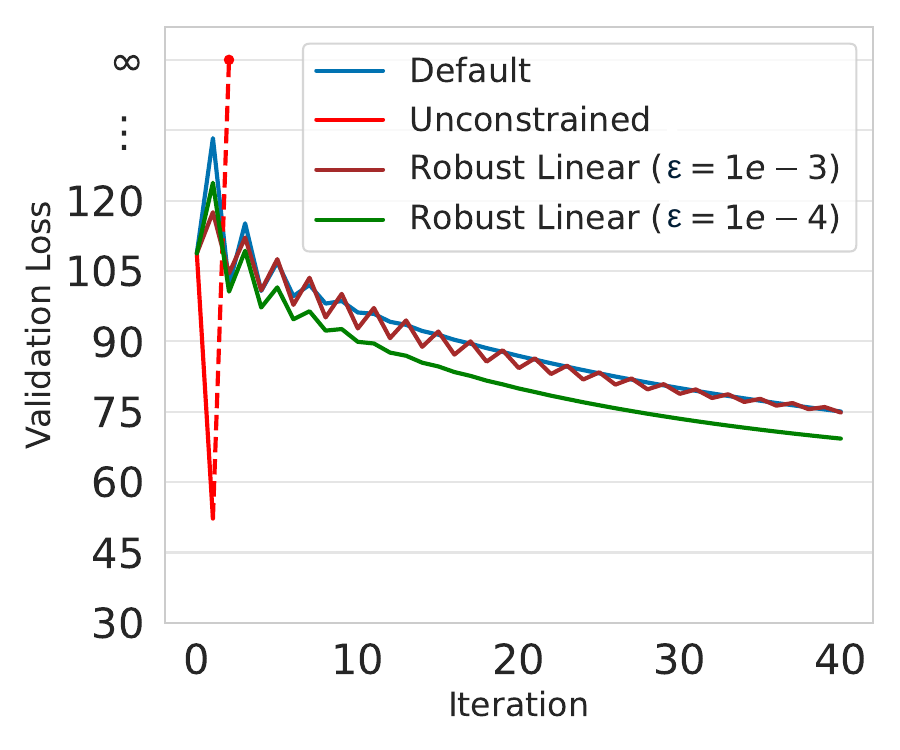}
    \end{subfigure}
    \caption{(Left) Distribution of theoretical robust weights for the linear case with $\epsilon=10^{-4}$, and (Right) validation loss during training with different variants in the running experiment setting.}
    \label{fig:robust-linear-eps}
\end{figure}

\section{Proof of Theorem \ref{cla:approx-grads}}
\label{apx:cluster-bound}
We begin by noting a property of the landmark-based approximation introduced in Section \ref{sec:landmark}: it exhibits \textit{rotational equivariance}. That is, if all source and target gradients are rotated by an orthonormal matrix, the resulting landmark-based gradient approximations will also be simply rotated by the same matrix.

In the remainder of this section, we prove two useful lemmas—Lemma \ref{lem:single-vect-sim} and Lemma \ref{lem:rotation-isotropic}. We then state and prove Theorem \ref{thm:approx-grad}, which bounds the error in the vector $\vp$ for any unbiased approximation that satisfies rotational equivariance. Finally, Corollary \ref{cor:approx-grads} bounds the difference in the resulting sample weights, thereby completing the proof of Theorem \ref{cla:approx-grads}.

\begin{lemma}
    \label{lem:single-vect-sim}
    Given unit vectors $\vg, \vt \mathbin{\in} \Roned{d}$, assume $\hat{\vg}\mathbin{=}\vg+\ve$ is a noisy approximation to $\vg$. Additionally, assume $\ve \mathbin{\in} \Roned{d}$ is a zero-mean, isotropic random vector, i.e., $\expected{\ve}=0$ and $\text{Cov}(\ve)=\sigma^2 \matr{I}_d$ for some $\sigma>0$. Let $S\mathbin{=}\langle \hat{\vg},\vt \rangle$.
    Then $\expected{S} \mathbin{=} \langle \vg,\vt \rangle$, and $\text{Var}(S)\mathbin{=}\frac{\expected{||\hat{\vg}-\vg||^2_2}}{d}$.
\end{lemma}
\begin{proof}
     The expectation of $S$ follows directly from zero-mean property of $\ve$. To bound its variance, let $\Sigma$ denote the covariance matrix of $\ve$. Since $\ve$ is isotropic, $\Sigma=\sigma^2 \matr{I}_d$ for some $\sigma$. 
     We can write:
    \begin{align*}
        \text{Var}(S) &= \text{Var} \langle \hat{\vg},\vt \rangle \\
        &= \text{Var} (\langle \vg,\vt \rangle + \langle \ve,\vt \rangle) \\
        &= \text{Var} (\langle \ve,\vt \rangle) \\
        &= \vt^T \Sigma \vt \\
        &= \sigma^2 ||\vt||^2 \\
        &= \sigma^2 \numberthis
    \end{align*}
    Also, 
    \begin{align*}
        \expected{||\hat{g} - g||^2} &= \expected{||\ve||^2} \\
        &= tr(\Sigma) \\
        &= d\sigma^2 \numberthis
    \end{align*}
    Hence $\sigma^2 = \frac{\expected{||\hat{g} - g||^2}}{d}$, which concludes the proof.
\end{proof}

\begin{lemma}
    \label{lem:rotation-isotropic}
    Assume $\vx \in \Roned{d}$ is a random vector from an arbitrary distribution. For any random orthonormal matrix of the form $\matr{R}=\matr{P} \matr{D}$, where 
    \begin{itemize}
        \item $\matr{P}$ is a random permutation matrix 
        \item $\matr{D}$ is a diagonal matrix with i.i.d. Rademacher signs ($\pm 1$)
    \end{itemize}
    the random vector $\vy = \matr{R} \vx$ is isotropic, i.e., $\text{Cov}(\vy)=\sigma^2 \matr{I}_d$ for some real value $\sigma$.
\end{lemma}
\begin{proof}
    We can write:
    \begin{align*}
        \text{Cov}(\matr{R} \vx) &= \mathbb{E}_{\matr{P}, \matr{D}, \vx}[\matr{R} \vx \vx^T \matr{R}^T] \\
        &= \mathbb{E}_{\matr{P}, \matr{D}, \vx}[\matr{P} \matr{D} \vx \vx^T \matr{D} \matr{P}^T] \\
        &= \mathbb{E}_{\vx}[\mathbb{E}_{\matr{P}}[\mathbb{E}_{\matr{D}} [\matr{P} \matr{D} \vx \vx^T \matr{D} \matr{P}^T ~|~ \matr{P}, \vx~] ~|~ \vx~]] \\
        &= \mathbb{E}_{\vx}[\mathbb{E}_{\matr{P}} [\matr{P} \mathbb{E}_{\matr{D}} [\matr{D} \vx \vx^T \matr{D} ~|~ \vx~] \matr{P}^T ~|~ \vx~]] \numberthis \label{eq:iso-lemma-eq1}
    \end{align*}
    Now note that $\mathbb{E}_{\matr{D}} [\matr{D} \vx \vx^T \matr{D} ~|~ \vx~] = \text{diag}(\vx^2)$. Substituting into the expectation over $\matr{P}$, we need to compute $\mathbb{E}_{\matr{P}} [\matr{P} \text{diag}(\vx^2) \matr{P}^T ~|~ \vx~]$. However, since $\matr{P}$ is a random permutation, off-diagonal elements are zero and for the diagonal elements, any element of $\vx^2$ can be picked with equal probability. Hence, the expectation over $\matr{P}$ equals $\frac{1}{d} ||\vx||^2 \matr{I}_d$.

    Putting it all together in Equation \ref{eq:iso-lemma-eq1}, we get 
    \begin{equation}
        \text{Cov}(\matr{R} \vx) = \frac{\expected{||\vx||^2}}{d} \matr{I}_d,
    \end{equation}
    which concludes the proof.
\end{proof}

\begin{theorem}
    \label{thm:approx-grad}
    Let $\mG \mathbin{\in} \Rtwod{n}{d}$ and $\vt \mathbin{\in} \Roned{d}$, with $\vt$ and each row of $\mG$ having unit lengths. Let $\vg_i$ denote the $i$'th row in $\mG$. Additionally, assume access to a (randomized) mapping function $\vh: \{\vg_1,\vg_2,\dots,\vg_n\} \mathbin{\rightarrow} \Roned{d}$, and let $\forall i \in \{1,2,\dots,n\}: \hat{\vg}_i\mathbin{=}\vh(\vg_i;\mG)$. Additionally, assume $\vh(.)$ satisfies:
    \begin{enumerate}
        \item Unbiased: $\forall i \in \{1,2,\dots,n\}: \expected{\hat{\vg}_i} = \vg_i$, i.e., $\vh(.)$ is unbiased.
        \item Bounded Average Mean Squared Error: Let $\delta_i^2 = \mathbb{E}[||\hat{\mathbf{g}}_i-\mathbf{g}_i||^2]$. Then:
        $$\frac{1}{n} \sum_{i=1}^{n} \delta_i^2 \le \Delta^2$$
        for some $\Delta^2 \ge 0$.
        \item Rotation Equivariance: For any orthonormal rotation matrix $\matr{R}\mathbin{\in}\Rtwod{d}{d}$ and $\forall i \in \{1,2,\dots,n\}: \vh(\matr{R} \vg_i; \mG \matr{R}) \mathbin{=} \matr{R} \hat{\vg}_i$.
    \end{enumerate}
    Construct the vector $\vp \mathbin{=} [p_1, p_2, \dots, p_n]^T$ such that $p_i\mathbin{=}\langle \vg_i, \vt \rangle$. Similarly define $\hat{\vp}\mathbin{=}[\hat{p}_i, \hat{p}_2, \dots, \hat{p}_n]^T$, where $\hat{p}_i\mathbin{=}\langle \hat{\vg}_i, \vt \rangle$. Then 
    \begin{equation}
        \expected{||\vp-\hat{\vp}||^2} \le \frac{n \Delta^2}{d}
    \end{equation}
\end{theorem}
\begin{proof}
    For all $i$, let $\ve_i \mathbin{=} \hat{\vg}_i - \vg_i$ denote the error. By the \textit{Unibased} assumption, $\expected{\ve_i}=\vect{0}$.

    Without loss of generality, we can assume that for all $i$, the vector $\ve_i$ is isotropic, i.e., $\text{Cov}(\ve_i)$ is a scalar multiple of the identity matrix. If this is not the case, we take advantage of Lemma \ref{lem:rotation-isotropic} and apply a change of variables: $\mG \leftarrow \mG \matr{R}$ and $\vt \leftarrow \matr{R} \vt$, where $\matr{R} = \matr{P} \matr{D}$, $\matr{P}$ is a permutation matrix, and $\matr{D}$ is a diagonal matrix with entries chosen uniformly at random from $\{\pm 1\}$. Note that by the \textit{Rotation Equivariance} assumption, this transformation implies $\hat{\vg}_i \leftarrow \matr{R} \hat{\vg}_i$. Under this transformation, the error vectors $\ve_i$ are mapped into a space where they become isotropic, and the pairwise dot products and distances remain unchanged as $\matr{R}$ is orthonormal.
    
    Now we can directly apply Lemma $\ref{lem:single-vect-sim}$ for each coordinate $i$: $\expected{\hat{p}_i}\mathbin{=}p_i$ and $\text{Var}(\hat{p}_i)=\frac{\expected{||\hat{g} - g||^2}}{d}$. This means:
    \begin{align*}
        \expected{||\vp-\hat{\vp}||^2} &= \sum_{i=1}^n \expected{(p_i - \hat{p_i})^2} \\
        &= \sum_{i=1}^n \text{Var}(\hat{p}_i) \\
        &= \sum_{i=1}^n \frac{\expected{||\hat{g} - g||^2}}{d} \\
        &\le \frac{n \Delta^2}{d}
    \end{align*}
    where the last inequality comes from the \textit{Bounded Average Mean Squared Error} assumption.
\end{proof}

\begin{corollary}
    \label{cor:approx-grads}
    In the setting of Theorem \ref{thm:approx-grad}, if we define:
    \begin{equation}
        \vw(\vp) = \argmin_{\vx}~ -\vp^T \vx + \frac{\lambda}{2} \|\vx\|_2^2,~~s.t.~ \left\{\begin{matrix} \vx \ge 0 \\ \vx^T \vone = n \\ \end{matrix}\right.
    \end{equation}
    then 
    \begin{equation}
        \expected{||\vw(\vp) - \vw(\hat{\vp})||^2} \le \frac{n \Delta^2}{\lambda^2 d}.
    \end{equation}
\end{corollary}
\begin{proof}
    Let $F_p(x)=-\vp^T \vx + \frac{\lambda}{2} \|\vx\|_2^2$ and $C = \{\vx \in \Roned{d}: \vx \ge 0, \vx^T \vone = n\}$.
    Note that the objective above has a unique solution since $F_p$ is $\lambda$-strongly convex and $C$ is a convex set independent of $\vp$.

    By strong convexity, $\forall x, y \in \Roned{d}$:
    \begin{equation}
        F_p(\vect{y}) \ge F_p(\vx) + \nabla_x F_p(\vx)^T (\vect{y} - \vx) + \frac{\lambda}{2} ||\vect{y}-\vx||^2
    \end{equation}
    Set $\vx = \vw := \vw(\vp)$ and $\vect{y} = \hat{\vw} := \vw(\hat{\vp})$. Since $\vw$ minimizes $F_p$ over $C$, $\nabla_x F_p(\vx)^T (\vect{y} - \vw) \ge 0$. Hence:
    \begin{equation}
        F_p(\hat{\vw}) \ge F_p(\vw) + \frac{\lambda}{2} ||\vw - \hat{\vw}||^2
    \end{equation}
    Swapping $\vw$ and $\hat{\vw}$,
    \begin{equation}
        F_{\hat{p}}(\vw) \ge F_{\hat{p}}(\hat{\vw}) + \frac{\lambda}{2} ||\vw - \hat{\vw}||^2
    \end{equation}
    Adding the two equations above:
    \begin{equation}
        (\vp - \hat{\vp})^T (\vw - \hat{\vw}) \ge \lambda ||\vw - \hat{\vw}||^2
    \end{equation}
    Applying Cauchy-Schwarz on the left hand side, we get
    \begin{equation}
        ||\vp - \hat{\vp}|| \cdot ||\vw - \hat{\vw}|| \ge \lambda ||\vw - \hat{\vw}||^2
    \end{equation}
    Hence
    \begin{equation}
        ||\vw - \hat{\vw}|| \le \frac{1}{\lambda} ||\vp - \hat{\vp}||
    \end{equation}
    Combining with the result of Theorem \ref{thm:approx-grad}:
    \begin{equation}
        \expected{||\vw - \hat{\vw}||^2} \le \frac{n \Delta^2}{\lambda^2 d}.
    \end{equation}
\end{proof}

\section{Dataset and Model Details}
\label{apx:datasets-models}

This section provides details on the datasets and models used throughout the paper.

\subsection{Datasets}
For the datasets, we largely follow the setup of \citet{ivison2025large}.

\paragraph{Tulu V2 (ODC-BY License).}
The Tulu V2 dataset \citep{tulu2}, also known as the Tulu V2 SFT Mixture, is a comprehensive instruction-tuning dataset. Following \citet{ivison2025large}, we consider the unfiltered version with 5.8M samples, consisting of 961,322 samples from FLAN v2 \citep{flan}, 398,439 samples from FLAN CoT \citep{flan}, 7,707 samples from Open Assistant \citep{openassistant}, 15,007 from Dolly \citep{dolly}, 52,002 from GPT-4 Alpaca \citep{gpt4-alpaca}, 20,022 from Code Alpaca \citep{codealpaca}, 100,054 from ShareGPT, 1,030 from LIMA \citep{lima}, 142,802 from Wizard Evol-Instruct V2 \citep{wizardlm}, 4,111,858 from Open Orca \citep{openorca}, 7,535 from SciRIFF \citep{sciriff}, and 14 from Hardcoded. For more information, we refer the reader to \citet{ivison2025large}.

\paragraph{MMLU (MIT License).}
The Massive Multitask Language Understanding (MMLU) dataset \citep{mmlu1,mmlu2} consists of challenging multiple-choice questions from 57 topics, such as abstract algebra, astronomy, machine learning, and more. It includes 5 development samples per category and a total of 14,042 test samples. We use the development samples as our target set and evaluate the final model zero-shot on the test set.

\paragraph{GSM8K (MIT License).}
This dataset comprises grade school math questions, with 7.47k training and 1.32k test samples \citep{gsm8k}. We evaluate the models on the test set using 8 examples in the context (8-shot evaluation) and use the same 8 individual samples as the target set. As is standard, only the final answer to each question is considered.

\paragraph{Big-Bench-Hard (MIT License).}
This dataset includes questions from 27 challenging tasks, such as causal judgment, multi-step arithmetic, and logic. Following \citet{bbh}, we perform 3-shot evaluations using the same 3 samples per category (a total of 81) as the target set.

\paragraph{TyDIQA (Apache-2.0 License).}
TyDIQA is a dataset of 204k question-answering samples across 11 languages \citep{tydiqa}. For evaluation, we follow \citet{ivison2025large}, which in turn follows \citet{anil2023palm}, using 1-shot prompting. We select 9 samples per language for the target set.

\paragraph{Codex (MIT License).}
This dataset contains 164 Python programming questions \citep{codex}, of which 16 are used as the target set and the remaining as the test set. See \citet{ivison2025large} for additional evaluation details.

\paragraph{SQuAD (CC BY-SA 4.0 License).}
The Stanford Question Answering Dataset (SQuAD) \citep{squad} contains reading comprehension questions based on Wikipedia articles. We use 500 random samples from the training split as the target set. We perform 3-shot evaluations with three samples randomly selected from the training set.

\subsection{Model Licenses}
In this paper, we utilize LLaMA 2 \citep{llama2}, LLaMA 3.2 3B \citep{llama3}, Qwen 2.5 1.5B \citep{qwen2.5}, and Qwen 2.5 3B \citep{qwen2.5} models. These models are distributed under the LLaMA 2 Community License, LLaMA 3.2 Community License, Apache-2.0 License, and Qwen Research License, respectively.

\section{Embeddings Study}
\label{apx:embds}
In Section \ref{sec:landmark}, we noted that existing embedding functions are insufficient for our landmark-based gradient approximations and introduced the JVP embeddings as an alternative. In this section, we compare different embedding functions in two settings. In all the experiments, the model we consider is Llama-2 7B \citep{llama2}.

\paragraph{Gradient Recovery.}
First, we randomly take 200k samples from Tulu V2 \citep{tulu2} and embed them using various embedding functions. We then use a small number of landmark gradient samples (selected uniformly at random) to approximate the gradients for all data points, following the method described in Section \ref{sec:landmark}. This process is repeated for different numbers of landmarks to evaluate how performance varies with landmark count. We report the average cosine similarity between the approximated gradients and the true gradients (projected into 8192-dimensional space using Rademacher-based projections \citep{ivison2025large, trak}) for each case.

We evaluate several embedding functions: the RDS+ embeddings from \citet{ivison2025large}, NVIDIA’s NV-Embed-v2 \citep{lee2024nv}, GTR-base \citep{gtr}, and our proposed JVP-based approach using two random vectors and four transformer blocks.

As a lower bound, we also include a \textit{Trivial} embedding: here, we assume that the gradients for the landmark samples are perfectly recovered, while the gradients for all other samples are treated as completely random.

Figure \ref{fig:embd-study} (Left) presents a comparison of these embedding functions. Our JVP embeddings outperform all other methods, including the more computationally intensive RDS+ and NV-Embed-v2.

Finally, we compute an upper bound by using the true projected gradients as the embedding function and repeating the same experiment. As shown in Figure \ref{fig:embd-study} (Right), this idealized setting quickly achieves high accuracy in gradient approximation—surpassing 0.9 cosine similarity with just over 4096 landmarks. This suggests that the gradients are approximately low-rank, a known phenomenon in LLMs \citep{hu2022lora, zhao2024galore}.

\begin{figure}[t] % The * makes it span both columns
    \centering
    \begin{subfigure}{0.47\textwidth}
        \includegraphics[width=\linewidth]{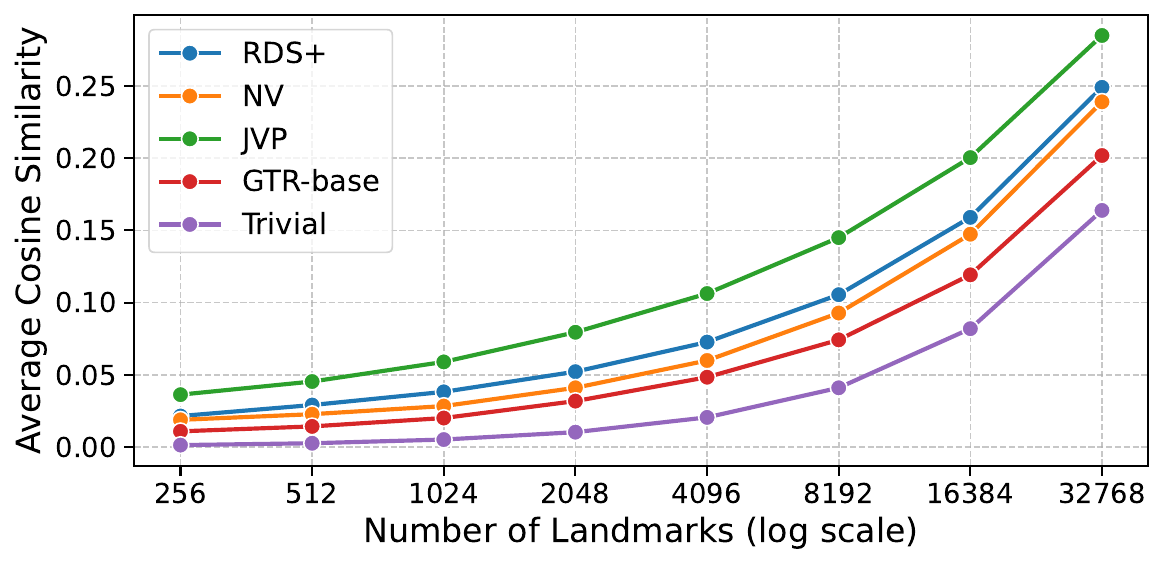}
    \end{subfigure}
    \begin{subfigure}{0.47\textwidth}
        \includegraphics[width=\linewidth]{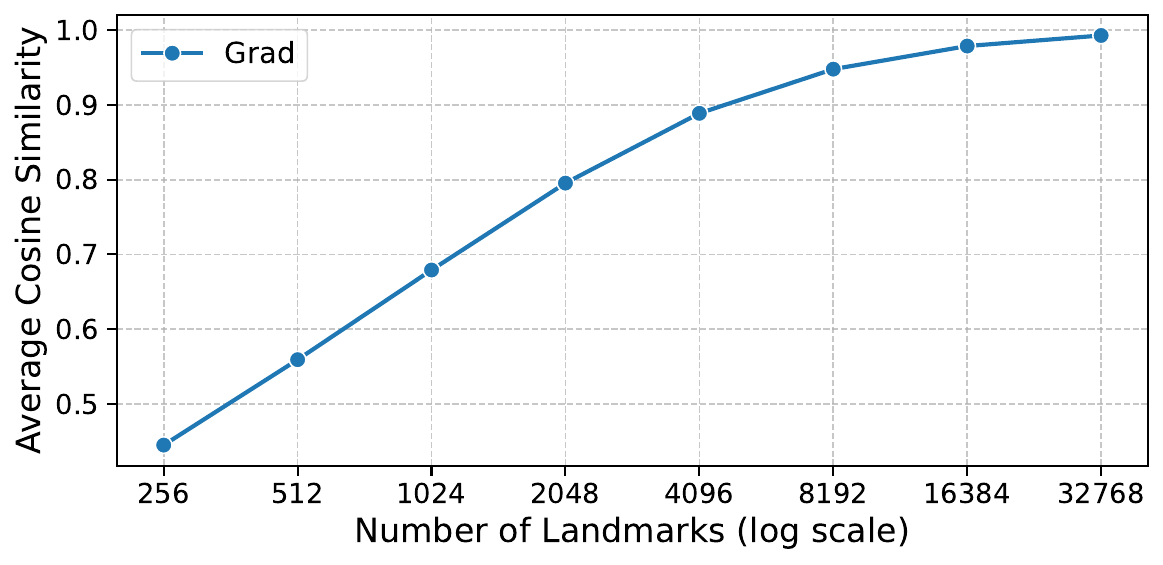}
    \end{subfigure}
    \caption{(Left) Gradient direction recovery vs number of landmarks, when different proxy embdeding functions are used, and (Right) gradient direction recovery when the actual gradients are used as an ideal embedding.}
    \label{fig:embd-study}
\end{figure}

\paragraph{End-to-end Selection and Training.}
We repeat the selection and fine-tuning experiments from Table \ref{tab:main-exp}, this time replacing the JVP embeddings with either GTR-base or true gradient embeddings. Table \ref{tab:e2e-embd} reports the resulting accuracy for each task. Due to the high computational cost of obtaining true gradients, we include only a single random seed for this setting.

We fix the number of landmarks to 4096 across all experiments. The results show that while GTR-base consistently underperforms, the JVP and true gradient embeddings yield comparable accuracy—falling within each other's standard deviation in most cases. This indicates that the gradient approximations provided by JVP embeddings are sufficiently accurate for end-to-end training.

Finally, we note that since Figure \ref{fig:embd-study} (Right) demonstrates near-perfect gradient recovery using the Grad embedding, the corresponding row in Table \ref{tab:e2e-embd} closely mirrors the performance of the LESS method \citep{less}.

\begin{table*}[t]
\centering
\caption{Accuracy ($\pm$ standard deviation) of Llama2-7B across six tasks when using \methodname{} with different embeddings to select 10k samples from a pool of 200k in the Tulu V2 dataset~\citep{tulu2}. The number of landmarks is fixed at 4096.}
\scalebox{0.65}{
\begin{tabular}{l|lcccccc|c}
\toprule
\textbf{Model} & \textbf{Embedding} & \textbf{MMLU} & \textbf{GSM8k} & \textbf{BBH} & \textbf{TyDIQA} & \textbf{CODEX} & \textbf{SQuAD} & \textbf{Avg. $\Delta$ w/ Uniform} \\
\midrule
\multirow{3}{*}{Llama2-7B} 
    & GTR-base  & 46.7 $\pm$ 0.17 & 18.7 $\pm$ 0.27 & 42.8 $\pm$ 0.34 & 52.2 $\pm$ 0.56 & 29.3 $\pm$ 0.84 & 82.1 $\pm$ 0.30 & 45.3 \\
    & JVP  & 48.3 $\pm$ 0.21 & 20.3 $\pm$ 1.65 & 43.2 $\pm$ 0.67 & 53.6 $\pm$ 0.34 & 29.5 $\pm$ 3.14 & 83.2 $\pm$ 1.02 & 46.4\\
    & Grad  & 48.3 & 20.2 & 43.7 & 51.7 & 27.7 & 84.5 & 46.0 \\
\bottomrule
\end{tabular}
}
\label{tab:e2e-embd}
\end{table*}

\section{An Active-Set Solution}
\label{apx:closed-form}
In this appendix we derive the solution to the \methodname{} objective under the assumption that $\eta \mQ + \lambda \matr{I}$ is positive definite (PD). This setting includes the special first-order case used in the main body of the paper, where $\eta \to 0$. Concretely, we solve
\begin{equation}
    \vw^* = \argmin_{\vw}~-\vp^T \vw + \frac{\eta}{2} \vw^T \mQ \vw + \frac{\lambda}{2} \vw^T \vw,
    ~~s.t.~ \left\{
        \begin{matrix}
            \vw \ge 0 \\
            \vw^T \vone = n \\
        \end{matrix}
    \right.
\end{equation}
where $n$ denotes the dimension of $\vw$ and $\eta \mQ + \lambda \matr{I} \succ \vzero$.

Introduce the Lagrange multipliers $\tau\in\mathbb{R}$ for the equality constraint and $\valpha \in \Roned{n}_{\ge 0}$ for the non-negativity constraints. The Lagrangian is
\begin{equation}
    L(\vw,\tau,\valpha)
    = -\vp^\top\vw
      + \frac{\eta}{2}\vw^\top\mQ\vw
      + \frac{\lambda}{2}\vw^\top\vw
      - \tau\bigl(\vone^\top\vw-n\bigr)
      - \valpha^\top\vw .
\end{equation}

Differentiating $L$ with respect to $\vw$ and setting it equal to zero yields
\begin{equation}
    \eta \mQ \vw + \lambda \vw - \vp - \tau \vone - \valpha = \vzero .
\end{equation}
Let $\mR \coloneqq \eta\mQ+\lambda\matr I \succ \vzero$. Then
\begin{equation}
\label{eq:kkt-stationary}
    \mR \vw - \vp - \tau \vone - \valpha = \vzero .
\end{equation}

By complementary slackness, $\forall i:\, \vw_i \valpha_i = 0$.  
Let $A \!=\!\{i \mid w_i = 0\}$ be the active set and $B$ its complement.  
Restricting~\eqref{eq:kkt-stationary} to the free indices gives
\begin{equation}
    \mR_{BB}\vw_{B} = \vp_{B} + \tau \vone_{B}.
\end{equation}

Because $\mR_{BB}$ is a principal sub-matrix of the PD matrix $\mR$, it is itself PD. Hence
\begin{equation}
    \vw(\tau;B) = \mR_{BB}^{-1} (\vp_{B} + \tau \vone_{B}) .
\end{equation}

Enforcing $\vone^T \vw = n$ determines $\tau$:
\begin{equation}
    \vone_{B}^T \mR_{BB}^{-1} (\vp_{B} + \tau \vone_{B}) = n ,
\end{equation}
and therefore
\begin{equation}
    \label{eq:kkt-tau}
    \tau^* = \frac{n-\vone_{B}^T \mR_{BB}^{-1} \vp_{B}}
                   {\vone_{B}^T \mR_{BB}^{-1} \vone_{B}} .
\end{equation}

Substituting $\tau^*$ back into $\vw(\tau;B)$ gives us the weights on $B$:
\begin{equation}
\label{eq:wB_star_continued}
    \vw_B^* = \mR_{BB}^{-1} \Bigl(
        \vp_{B}
      + \bigl(
            \frac{n-\vone_{B}^T \mR_{BB}^{-1} \vp_{B}}
                 {\vone_{B}^T \mR_{BB}^{-1} \vone_{B}}
        \bigr) \vone_{B}
    \Bigr) .
\end{equation}
For indices in the active set $A$ we have $\vw_A^* = \vzero$, giving the final candidate solution $\vw^* = (\vw_A^*, \vw_B^*)$.

Optimality requires that the remaining Karush–Kuhn–Tucker (KKT) conditions hold, namely
$\forall i\in B,\ \vw_i \ge 0$ (primal feasibility) and
$\forall j\in A,\ \valpha_j \ge 0$ (dual feasibility).
Because the objective is convex ($\mR \succ 0$), any partition $A,B$ satisfying these conditions is the global optimum.

Examining the coordinates in $A$ in~\eqref{eq:kkt-stationary} gives
\begin{equation}
    \label{eq:kkt-alpha}
    \valpha_A = (\mR_{AB} \vw_B^*)_A - \vp_A - \tau^* \vone_A .
\end{equation}

Problems of this type are typically solved with a primal–dual active-set algorithm.
We start from the feasible point $\vw=\vone$ (so $A=\varnothing$, $B=\{1,\ldots,n\}$) and repeat:

\begin{enumerate}
    \item Solve for $\vw_B^*$ via~\eqref{eq:wB_star_continued}.  
    \item If any component of $\vw_B^*$ is negative, move its index to $A$.  
    \item Compute $\valpha_A$; if any component is negative, move its index back to $B$.  
\end{enumerate}

Each move strictly decreases the objective, and with only finitely many index sets the algorithm terminates once all components of $\vw_B$ and $\valpha_A$ are non-negative.

\paragraph{The Special Case of $\eta \rightarrow 0$.}
This setting corresponds to the first-order \methodname{} variant used throughout the main body of the paper. In this case, we demonstrate that as $\lambda$ increases, the solution $\vw^*$ becomes denser—that is, it contains more non-zero elements. This observation is leveraged in Section \ref{sec:tuning-lambda} for tuning the parameter $\lambda$.

When $\eta \rightarrow 0$, we can write $\mR = \lambda \mI$, which implies $\mR_{BB}^{-1} = \frac{1}{\lambda} \mI$ and $\mR_{AB} = \vzero$. Substituting these into Equations \ref{eq:kkt-tau}, \ref{eq:wB_star_continued}, and \ref{eq:kkt-alpha}, we obtain:

\begin{equation}
\tau^* = \frac{n \lambda - \vone_{B}^T \vp_{B}}{|B|}
\end{equation}

\begin{equation}
\vw_B^* = \frac{1}{\lambda}(\vp + \tau^* \vone)_{B}
\end{equation}

\begin{equation}
\valpha_A = - (\vp + \tau^* \vone)_A
\end{equation}

Since both $\vw_B$ and $\valpha_A$ must be non-negative, the last two equations imply that the active set $B$ must satisfy $B = \{i : \vp_i \ge -\tau^*\}$, i.e., $B$ is necessarily a set of top-k elements from $\vp$ for some $k$.

Consider two values $\lambda_1 < \lambda_2$, and let $B_1$ and $B_2$ denote their optimal supports with sizes $k_1$ and $k_2$, and $\vw^{(1)}$ and $\vw^{(2)}$ their respective optimal weight vectors; similarly, let $\valpha^{(1)}$ and $\valpha^{(2)}$ denote their associated dual variables. Suppose for contradiction that $k_2 < k_1$. Note that $B_1$ consists of the indices of the top $k_1$ elements in $\vp$, while $B_2 \subset B_1$ includes the top $k_2$ elements of $\vp$. Let $s_{k_1}$ and $s_{k_2}$ represent the sums of the top $k_1$ and $k_2$ elements in $\vp$, respectively. Define $j$ as the index of the $k_1$-th largest element in $\vp$. Since $j \in B_1$, we have $\vw^{(1)}_j \ge 0$, and since $j \notin B_2$, it follows that $\valpha^{(2)}_j \ge 0$. Therefore,
\begin{align*}
    & \vw^{(1)}_j \ge 0 \\
    \Rightarrow~& \vp_j + \frac{n \lambda_1 - s_{k_1}}{k_1} \ge 0 \\
    \Rightarrow~& n \lambda_1 \ge s_{k_1} - k_1 \vp_j = \sum_{i \in B_1} (\vp_i - \vp_j)
\end{align*}
and
\begin{align*}
    & \valpha^{(2)}_j \ge 0 \\
    \Rightarrow~& \vp_j + \frac{n \lambda_2 - s_{k_2}}{k_2} \le 0 \\
    \Rightarrow~& n \lambda_2 \le s_{k_2} - k_2 \vp_j = \sum_{i \in B_2} (\vp_i - \vp_j)
\end{align*}
Observe that $\sum_{i \in B_2} (\vp_i - \vp_j) \le \sum_{i \in B_1} (\vp_i - \vp_j)$ by definition of $\vp_j$, leading to the inequality $n \lambda_2 \le n \lambda_1$, which contradicts our initial assumption that $\lambda_1 < \lambda_2$.

This contradiction confirms that as the regularization parameter $\lambda$ increases, the solution becomes progressively denser. Specifically, at $\lambda = 0$, the solution concentrates all weight on the largest element of $\vp$ to minimize the objective, whereas in the limit as $\lambda \rightarrow \infty$, the regularization dominates, resulting in $\vw = \vone$.

\section{First- vs Second-Order \methodname{}}
\label{apx:first-vs-second-order}
Recall the robust Objective \ref{eq:robust-obj}
\begin{equation}
\vw^* = \argmin_{\vw}~ f(\vw;\vtheta) + \frac{\lambda}{2} \|\vw\|_2^2,~~s.t.~ \left\{\begin{matrix} \vw \ge 0 \\ \vw^T \vone = |\dS| \\ \end{matrix}\right.
\label{eq:robust-obj2}
\end{equation}
where,
\begin{equation}
\label{eq:f2}
f(\vw;\vtheta) \mathbin{=} -\vp(\vtheta)^T \vw + \frac{\eta}{2} \vw^T \mQ(\vtheta) \vw
\end{equation}
In this section, we compare the first-order term $T_1 = \vp(\vtheta)^T \vw$ with the second-order term $T_2 = \frac{\eta}{2} \vw^T \mQ(\vtheta) \vw$. To do so, we sample 128 random examples from the Tulu V2 dataset \citep{tulu2} as the source dataset, and 4 examples from either GSM8k \citep{gsm8k} or MMLU \citep{mmlu1, mmlu2} as the target dataset.

We compute the vectors $\vp$ and the matrices $\mQ$ exactly for the Qwen-2.5 1.5B model \citep{qwen2.5}, using Hessian-vector products to obtain $\mQ$. We then evaluate both $T_1$ and $T_2$ using default weights $\vw = \vone$ and a range of learning rates. To measure the relative contribution of the second-order term, we report the ratio $\left|\frac{T_2}{T_1}\right|$.

As shown in Figure~\ref{fig:second_vs_first_order}, the second-order term is generally negligible for practical learning rates ($\eta \le 10^{-4}$), indicating that the first-order approximation is sufficient in this setting.

\begin{figure}[t] % The * makes it span both columns
    \centering
    \includegraphics[width=0.75\linewidth]{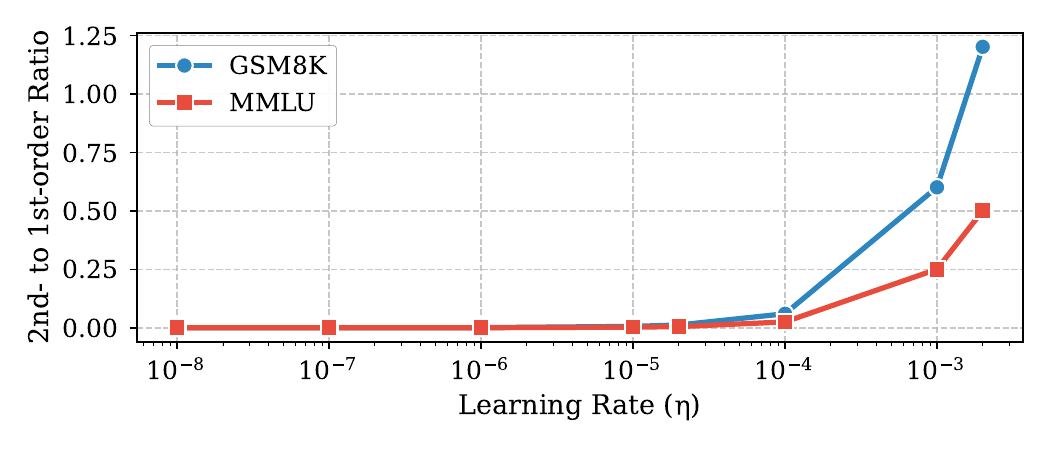}
    \caption{Ratio of second- to first-order terms for Qwen-2.5 1.5B across learning rates on two target datasets.}
    \label{fig:second_vs_first_order}
\end{figure}

\section{Projection Details}
\label{apx:proj}
While in some of our lower-cost experiments we employ Rademacher-based projections—including projecting JVP embeddings to a 4096-dimensional space using this method, as supported on GPUs by \citet{trak}—we find that projecting the landmark gradients with Rademacher projections becomes a computational bottleneck. To address this, we instead use a combination of pre-masking and Randomized Hadamard Transform-based projections, as described below.

\paragraph{Hadamard-based Projection.}
Given a high-dimensional gradient vector $\vg$, we first pad it with zeros to the nearest power of two, $2^k$. Then, we apply a random sign ($\pm 1$) to each element. The signed vector is reshaped into a matrix $\matr{X}$ of dimensions $m = 2^{\left\lceil \frac{k}{2} \right\rceil}$ and $n = 2^{\left\lfloor \frac{k}{2} \right\rfloor}$. We then apply Hadamard transforms from both sides: $\mH_m^T \matr{X} \mH_n$. The resulting matrix is flattened, and a random subset of its entries is selected as the projected vector.

Importantly, both the random sign patterns and the final index subset are generated once and reused across all projected vectors. This ensures consistency and enables meaningful comparison. The left and right Hadamard transforms are highly efficient and provide strong mixing across rows and columns.

\paragraph{Pre-masking.}
Although efficient GPU implementations of the Hadamard transform exist \citep{hadacore, tridao}, they support transforms up to dimension $2^{15} = 32{,}768$. This allows us to efficiently project vectors of up to $2^{30} = 1{,}073{,}741{,}824$ elements—just over one billion. However, the full gradients of large language models (LLMs) can exceed this size.

To address this, we apply \emph{pre-masking}: we randomly select one billion elements from the gradient vector before projection. For LLaMA-2 7B \citep{llama2}, we select these elements from the \texttt{down\_proj} matrices, which we find to represent the overall gradients well. For smaller models, we randomly sample one billion elements from the entire gradient vector.

\section{Weighted Training Loss}
\label{apx:weighted-loss}

\begin{table*}[t]
\centering
\caption{Accuracy ($\pm$ standard deviation) of Llama2-7B across six tasks when using \methodname{} to select 10k samples from a pool of 200k in the Tulu V2 dataset~\citep{tulu2}, with and without loss weighting during training. The number of landmarks is fixed at 8192.}
\scalebox{0.65}{
\begin{tabular}{l|lcccccc|c}
\toprule
\textbf{Model} & \textbf{Embedding} & \textbf{MMLU} & \textbf{GSM8k} & \textbf{BBH} & \textbf{TyDIQA} & \textbf{CODEX} & \textbf{SQuAD} & \textbf{Avg. $\Delta$ w/ Uniform} \\
\midrule
\multirow{2}{*}{Llama2-7B} 
    & Weighted  & 47.8 $\pm$ 0.16 & 19.5 $\pm$ 0.06 & 42.3 $\pm$ 0.26 & 52.2 $\pm$ 1.38 & 27.0 $\pm$ 2.53 & 84.4 $\pm$ 0.48 & +1.48\\
    & Not Weighted  & 48.2 $\pm$ 0.35 & 19.6 $\pm$ 0.79 & 42.4 $\pm$ 0.14 & 52.7 $\pm$ 1.67 & 29.3 $\pm$ 1.27 & 83.4 $\pm$ 0.86 & +1.93 \\
\bottomrule
\end{tabular}
}
\label{tab:weighted-loss}
\end{table*}

In this section, we investigate the effect of incorporating the weights derived by \methodname{} into the training loss. Specifically, we conduct an experiment using LLaMA-2 7B \citep{llama2}, with a pool size of 200k and 8192 landmarks sampled from Tulu V2 \citep{tulu2}. During training, we scale the loss of each selected sample by its corresponding weight.

Table~\ref{tab:weighted-loss} compares this weighted training setup with a baseline where the weights of the selected samples are ignored. The results show that incorporating weights during training does not improve performance—and in some cases, it may even degrade it. This may be due to some samples having near-zero weights, effectively pruning them from the training process.

\section{Differed Figures}

\begin{figure}[t]
    \centering
    \includegraphics[width=\linewidth]{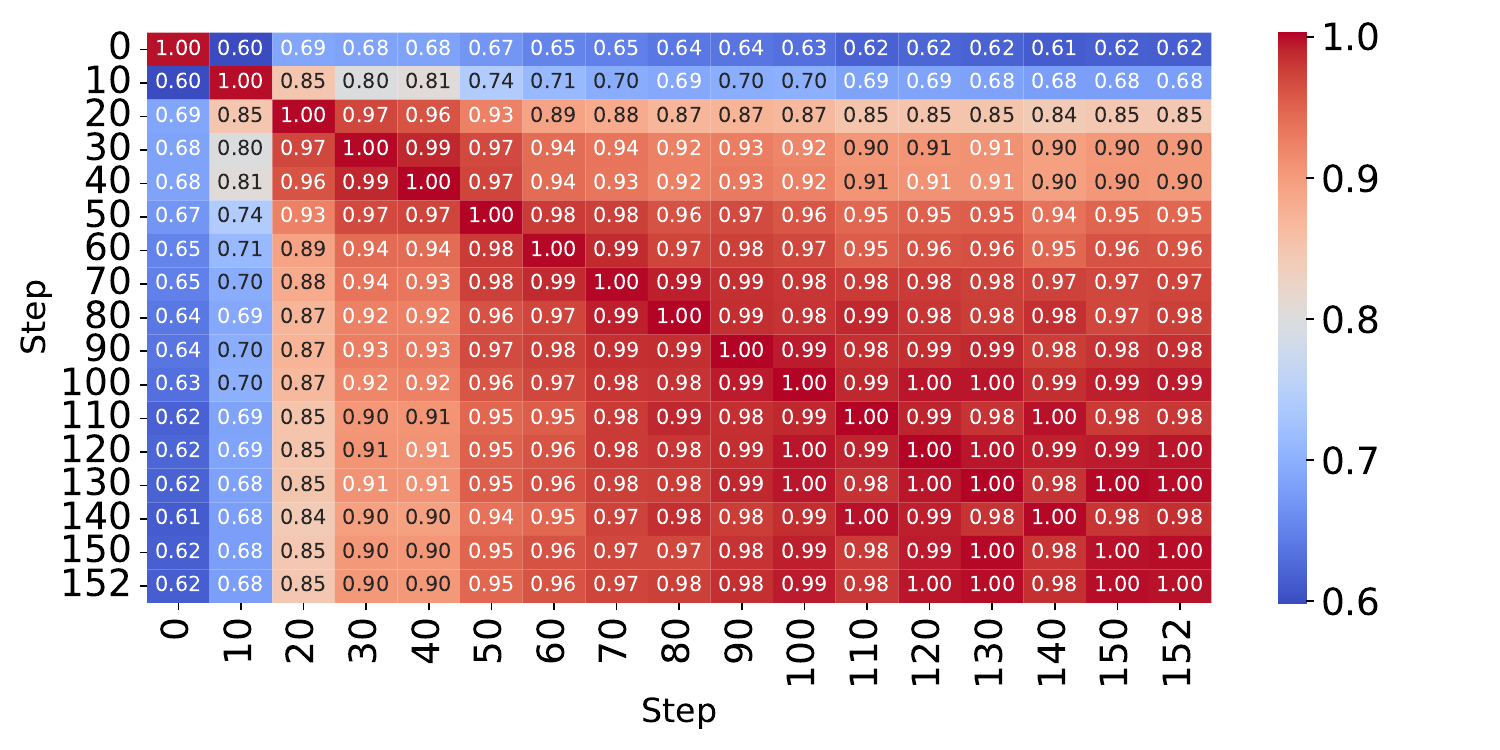}
    \includegraphics[width=\linewidth]{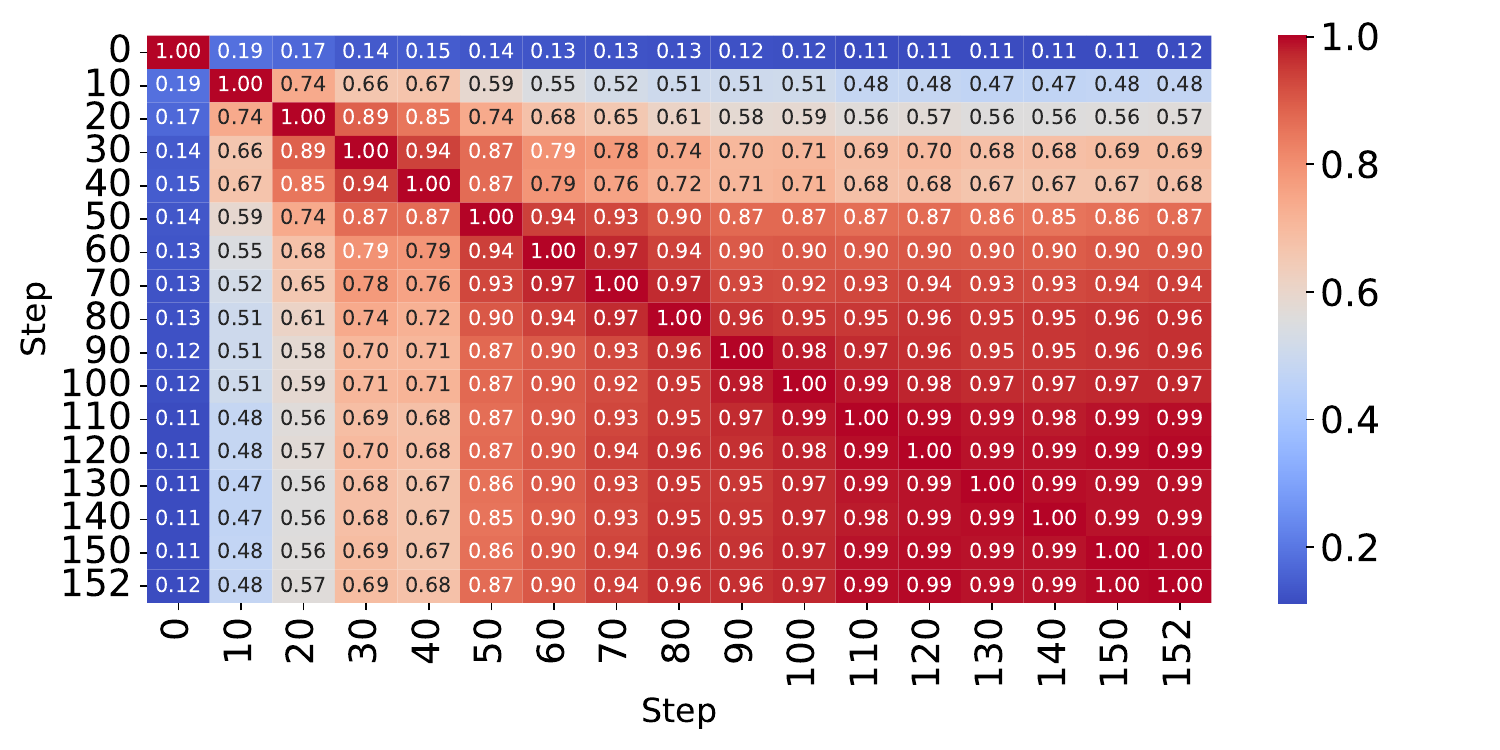}
    \caption{Average gradient cosine similarity on unseen samples from GSM8k (top) and SQuAD (bottom) across checkpoints.}
    \label{fig:grad-sim-other}
\end{figure}

\end{document}